\documentclass{article} %
\usepackage{iclr2027_conference,times}

\usepackage{hyperref}
\usepackage{url}

\title{Training Mixed-Monotone Neural Networks for Verification Without Set Propagation}

\author{Jonas Fingerhut,\quad
  Lukas Koller,\quad
  Matthias Althoff\\
  Technical University of Munich \\
  \texttt{\{jonas.fingerhut, lukas.koller, althoff\}@tum.de}
}

\usepackage[sets,operations,nn,colors,tikz]{cora-macs}
\usetikzlibrary{fadings,shadings,arrows.meta} %

\usepackage{amsmath}
\usepackage{amsthm}
\usepackage{amssymb}
\usepackage{mathtools}

\newtheorem{definition}{Definition}
\newtheorem{proposition}{Proposition}

\newtheorem{example}{Example}

\newtheorem*{rep@theorem}{\rep@title}
\newcommand{\newreptheorem}[2]{%
  \newenvironment{rep#1}[1]{%
    \def\rep@title{#2 \ref{##1}}%
    \begin{rep@theorem}}%
      {\end{rep@theorem}}}
\makeatother
\newreptheorem{proposition}{Proposition}
\newreptheorem{corollary}{Corollary}
\newreptheorem{theorem}{Theorem}

\let\originalleft\left
\let\originalright\right
\renewcommand{\left}{\mathopen{}\mathclose\bgroup\originalleft}
\renewcommand{\right}{\aftergroup\egroup\originalright}

\usepackage{nicefrac, xfrac}

\usepackage{aligned-overset}

\usepackage[inline]{enumitem}

\usetikzlibrary{fit,positioning,calc,matrix,arrows.meta,decorations.pathreplacing}

\usepackage{cleveref}
\crefname{figure}{Fig.}{Fig.}
\Crefname{figure}{Fig.}{Fig.}
\crefname{section}{Sec.}{Sec.}
\Crefname{section}{Sec.}{Sec.}
\crefname{equation}{}{}
\Crefname{equation}{}{}
\crefname{definition}{Def.}{Def.}
\Crefname{definition}{Def.}{Def.}
\crefname{proposition}{Prop.}{Prop.}
\Crefname{proposition}{Prop.}{Prop.}
\crefname{corollary}{Col.}{Col.}
\Crefname{corollary}{Col.}{Col.}
\crefname{theorem}{Thm.}{Thm.}
\Crefname{theorem}{Thm.}{Thm.}
\crefname{table}{Tab.}{Tab.}
\Crefname{table}{Tab.}{Tab.}
\crefname{algorithm}{Alg.}{Alg.}
\Crefname{algorithm}{Alg.}{Alg.}
\crefname{line}{line}{lines}
\Crefname{line}{line}{lines}
\crefname{example}{Ex.}{Ex.}
\Crefname{example}{Example}{Examples}
\crefrangeformat{line}{lines #3#1#4--#5#2#6}
\crefname{appendix}{Appendix}{Appendix}
\Crefname{appendix}{Appendix}{Appendix}
\crefformat{footnote}{#2\footnotemark[#1]#3}

\usepackage{float}
\usepackage{booktabs}
\usepackage{makecell}
\usepackage{multirow}
\usepackage{booktabs} %
\usepackage{etoolbox}
\robustify\bfseries
\usepackage{amstext} %
\usepackage{array} %
\newcolumntype{L}{>{$}l<{$}} %
\newcolumntype{R}{>{$}r<{$}} %

\usepackage{lipsum}
\usepackage{placeins}

\tikzexternaldisable

\newcommand{\cmat}[1]{\begin{bmatrix} #1 \end{bmatrix}} %

\newcommand{\abs}[1]{\left|#1\right|}%
\newcommand{\zeros}{\ensuremath{\mathbf{0}}}
\newcommand{\ones}{\ensuremath{\mathbf{1}}}

\newcommand{\Weights}{\ensuremath{W}} %
\newcommand{\bias}{\ensuremath{b}} %

\DeclareMathOperator{\relu}{ReLU}

\newcommand{\numGens}{\ensuremath{q}} %

\newcommand{\bigO}{\ensuremath{\mathcal{O}}}

\newcommand{\norm}[1]{\lVert#1\rVert}

\NewDocumentCommand{\nablaOperator}{e{_^}}{%
  \mathop{}\!%
  \nabla
  \IfValueT{#1}{_{\!#1}}%
  \IfValueT{#2}{^{#2}}%
}

\newcommand{\monoInc}{\ensuremath{g}}          %
\newcommand{\monoDec}{\ensuremath{h}}          %

\NewDocumentCommand{\monoNNInc}{o}{\ensuremath{\NN\IfValueT{#1}{#1}_{+}}}
\NewDocumentCommand{\monoNNDec}{o}{\ensuremath{\NN\IfValueT{#1}{#1}_{-}}}
\newcommand{\decompFun}[1]{\ensuremath{d}_{#1}}        %

\newcommand{\boxlo}{\ensuremath{\underline{\nnInput}}}
\newcommand{\boxhi}{\ensuremath{\overline{\nnInput}}}
\newcommand{\size}[1]{\ensuremath{\norm{#1}_\mathrm{I}}}%

\newcommand{\lipBudget}{\ensuremath{\lambda}}
\newcommand{\widthWeight}{\ensuremath{\beta}}

\newcommand{\monVar}[2]{\ensuremath{\nu_{#1}\left(#2\right)}}

\crefname{corollary}{Cor.}{Cors.}
\Crefname{corollary}{Cor.}{Cors.}
\crefname{proposition}{Prop.}{Props.}
\Crefname{proposition}{Prop.}{Props.}
\crefname{remark}{Rem.}{Rems.}
\Crefname{remark}{Rem.}{Rems.}
\newcommand{\papertablestyle}[1][6pt]{%
  \small
  \setlength{\tabcolsep}{#1}%
  \renewcommand{\arraystretch}{1.08}%
}

\iclrfinalcopy %
\begin{document}

\maketitle
\lhead{Preprint}

\begin{abstract}
Most formal verification approaches for neural networks conservatively enclose the output set of the neural network. 
However, computing a tight output enclosure is in general computationally hard.
We sidestep this challenge with a novel neural network architecture that realizes mixed-monotone functions as the difference of monotone subnetworks.
By utilizing the theory on mixed-monotone systems, this reduces the computation of a tight output enclosure from complicated set propagation to two simple forward passes of opposite corners of an input interval. 
Moreover, because the size of the resulting enclosure is a closed-form function of the weights of the neural network, we can directly minimize it during training.
We prove a lower bound on the size of any such enclosure, given by the total variation of the represented function generalized to several inputs, which makes it possible to compare the tightness of the enclosure against the optimum.
Our evaluation on regression and dynamical-system benchmarks demonstrates that our training reduces the enclosure size to within {$0.5\,\%$} of the lower bound without sacrificing prediction accuracy.
Computing our enclosure is orders of magnitude faster than CROWN and $\alpha$-CROWN and on deep networks it is also tighter than zonotope propagation and CROWN, while its size exceeds that of $\alpha$-CROWN by only {$5.1\,\%$}.
\end{abstract}

\section{Introduction}

Formal verification of neural networks can decide whether the output of a neural network meets a given specification for all inputs bounded by a set.
Especially when learned components are deployed in safety-critical control loops, for example in autonomous driving~\citep{ye_2022}, formal verification is essential.
The exact image of a neural network is in general non-convex, and complete verifiers based on SMT or mixed-integer programming scale poorly in the number of neurons~\citep{katz_2017}.
Practical tools therefore enclose that image by propagating sets layer by layer, e.g., using intervals~\citep{gowal_2019}, zonotopes~\citep{singh_2018,koller_et_al_2025_shadow}, linear relaxations~\citep{zhang_2018} or polynomial zonotopes~\citep{kochdumper_althoff_2021,kochdumper_2023}, adding an abstraction error at every nonlinear layer.
That abstraction error accumulates with the depth of the neural network and with the size of the input set.
The abstraction error can be reduced with additional computations, e.g., higher-order abstractions~\citep{kochdumper_2023,ladner_althoff_2023} or branch-and-bound methods~\citep{bunel_et_al_2020}, making it possible to trade off small abstraction errors against high computational cost.

\paragraph{Main idea}
\begin{figure}
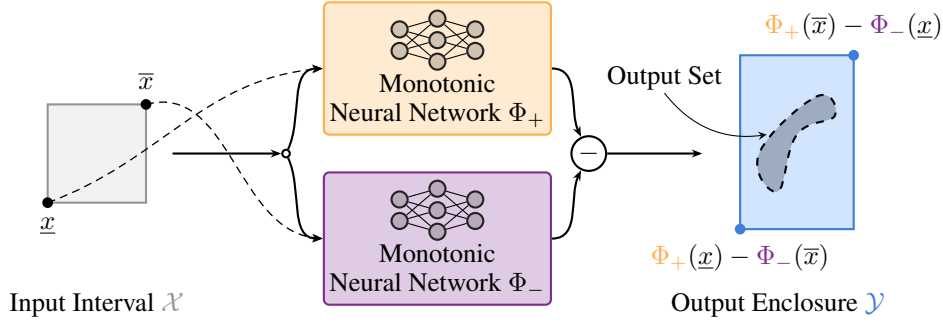

  \centering
  \includetikz{architecture}
  \caption{Our neural network architecture is the difference of two monotone subnetworks~$\monoNNInc$ and~$\monoNNDec$. Evaluating both subnetworks at the two opposite corners of the input interval yields an output enclosure without set propagation.}
  \label{fig:architecture}
\end{figure}
 
We overcome this tradeoff by constructing the neural network as a mixed-monotone function, which we realize as the difference of two monotone subnetworks~(\cref{fig:architecture}).
This allows us to use the theory on mixed-monotone systems to enclose the image of the neural network for an input interval without any set propagation. 
We only require two forward passes at opposite corners of the input interval, independent of the depth and width of the neural network and of the input dimension.

Moreover, we use the training of the neural network to explicitly optimize the tightness of the computed enclosure.
Our results show that training reduces the enclosure size to within $0.5\,\%$ of a lower bound that we prove, without sacrificing prediction accuracy.
The computation of our enclosure is (i) orders of magnitude faster than CROWN and $\alpha$-CROWN and (ii) on deep neural networks it is also tighter than zonotope propagation and CROWN, while its size exceeds that of $\alpha$-CROWN by only {$5.1\,\%$}.

\paragraph{Related work}
Monotonicity of neural networks can be achieved using sign-constrained weights~\citep{sill_1997,daniels_velikova_2010}, lattice networks~\citep{you_2017}, or certification of monotonicity after training~\citep{liu_2020,sivaraman_2020}; we build on the sign-constrained neural networks of~\citet{runje_2023}.
That literature is motivated by fairness, interpretability, and inductive bias; to the best of our knowledge, the enforced monotonicity has not been exploited for formal verification.

Mixed-monotone approaches combine a decomposition of the plant with a separately computed inclusion function for the network~\citep{jafarpour_2023}, or decompose a given neural network~\citep{meyer_2022}. In both cases, the decomposition is obtained after the neural network is fixed, and its tightness is inherited from the used verifier.
Classically, finding tight mixed-monotone decompositions is hard and the known constructions require structural assumptions, such as a sign-stable Jacobian~\citep{yang_2019,khajenejad_2023}.
In contrast, our neural network has a decomposition function by construction, and we directly optimize the tightness of the resulting enclosure during training, because the size of the enclosure is differentiable with respect to the weights.

Since the tightness of the enclosure is part of our training objective, a related field is certified training~\citep{gowal_2019,xu_2020,mueller_et_al_2023_sabr}, where abstraction errors are implicitly reduced by optimizing for robustness using sets that are propagated during training.
Further, set-based training~\citep{koller_et_al_2025_tmlr,wendl_et_al_2026_tmlr} is even closer to our work, as it explicitly penalizes the size of the enclosure in the loss, but it still requires a differentiable set propagation in every training step, as required for any other certified training method.

\paragraph{Contributions}
Our main contributions are:
\begin{itemize}[leftmargin=*,topsep=2pt,itemsep=1pt]
  \item A neural network architecture for mixed-monotone functions, allowing one to compute an output enclosure using two simple forward passes.
  \item A formula for the smallest enclosure achievable by a difference of two monotone functions: we prove that its size equals the total variation of the represented function generalized to several inputs.
  \item A neural network training objective that directly minimizes the size of the enclosure.
  \item An evaluation on static analytic models and dynamical systems, including the enclosure of reachable sets over multiple time steps and comparisons with zonotope propagation, interval bound propagation (IBP), CROWN, and $\alpha$-CROWN.
\end{itemize}

\section{Preliminaries}
\label{sec:prelim}

\paragraph{Notation}
We denote vectors with lowercase letters and matrices with uppercase letters.
The $i$-th entry of a vector $x$ is written as $x_{(i)}$ and the entry in the $i$-th row and the $j$-th column of a matrix $A\in\R^{n\times m}$ is written as $A_{(i,j)}$.
For vectors $x,y\in\R^{n}$ of equal dimension, we write $x \leq y$ for the component-wise partial order, i.e., $x_{(i)} \leq y_{(i)}$ for all $i\in\{1,2,\dots,n\}$.
We write the matrix (with appropriate size) that contains only zeros or ones as \zeros{} and \ones{}.
Given two matrices $A\in\R^{m\times n_1}$ and $B\in\R^{m\times n_2}$, their (horizontal) concatenation is denoted by $\cmat{A \ B}\in\R^{m\times (n_1 + n_2)}$.
The operation $\abs{A}$ takes the elementwise absolute value of a matrix $A\in\R^{n\times m}$; we use $\sign{\cdot}$ likewise.
Sets are denoted by uppercase calligraphic letters.
The set of all natural numbers up to $n\in\N$ is written as $[n]=\{1,2,\dots,n\}$.
For $a,b\in\R^m$ with $a\leq b$, we write $\shortI{a}{b} = \{x\in\R^m \mid a\leq x\leq b\}$ for the interval between $a$ and $b$.
For a nonempty bounded set $\contSet{S}\subseteq\R^m$, its interval hull $\shortI{\underline{s}}{\overline{s}}$ is the smallest interval containing $\contSet{S}$, where $\underline{s}_{(i)} = \inf_{s\in\contSet{S}} s_{(i)}$ and $\overline{s}_{(i)} = \sup_{s\in\contSet{S}} s_{(i)}$ for all $i\in[m]$.
Similar to~\citet[Sec.~3.1]{althoff_2023}, we measure the size of $\contSet{S}$ by the sum of the edge lengths of its interval hull, $\size{\contSet{S}} = \ones^\top(\overline{s}-\underline{s})$, which is referred to as the interval norm.

\subsection{Monotone functions and mixed-monotone decompositions}
\label{sec:mixmon}

We utilize the theory of mixed monotonicity to construct neural networks that admit easy-to-compute and tight output enclosures.
A monotone function returns its extremal values at two opposite corners of an input interval.
\begin{definition}[Monotone function, {\citep[Def.~1]{coogan_arcak_2015}}]
  \label{def:monotone}
  Given an interval $\nnInputSet_0\subseteq\R^{n}$, a function $f\colon \nnInputSet_0 \to \R^{m}$ is
  \begin{enumerate}[label=(\roman*)]
    \item non-decreasing on $\nnInputSet_0$ if $\nnInput \leq \nnInput' \implies f(\nnInput) \leq f(\nnInput')$, for all $\nnInput,\nnInput' \in \nnInputSet_0$, and
    \item non-increasing on $\nnInputSet_0$ if $\nnInput \leq \nnInput' \implies f(\nnInput) \geq f(\nnInput')$, for all $\nnInput,\nnInput' \in \nnInputSet_0$.
  \end{enumerate}
\end{definition}
Throughout this work, a monotone (or monotonic) function or neural network refers to a non-decreasing one.

Mixed monotonicity generalizes monotonicity by allowing a function to be non-decreasing in some variables and non-increasing in others.
\begin{proposition}[Enclosing a mixed-monotone function, {\citep[Def.~2]{coogan_arcak_2015}}] \label{def:decfun}\label{prop:soundness}
  For $\nnInputSet_0\subseteq \R^{n}$, let $f\colon \nnInputSet_0 \to \R^{m}$.
  A function $\decompFun{f}\colon \nnInputSet_0 \times \nnInputSet_0 \to \R^{m}$ is a decomposition function for $f$ iff for all $\nnInput_1,\nnInput_2 \in \nnInputSet_0$
  \begin{align*}
    \text{(i)}\;   & \decompFun{f}(\nnInput_1,\nnInput_1) = f(\nnInput_1) \text{,} &
    \text{(ii)}\;  & \decompFun{f}(\cdot,\nnInput_2) \text{ is non-decreasing,}    &
    \text{(iii)}\; & \decompFun{f}(\nnInput_1,\cdot) \text{ is non-increasing.}
  \end{align*}
 The function $f$ is called mixed-monotone if a decomposition function for $f$ exists.
  The image of $f$ over any interval $\nnInputSet = \shortI{\boxlo}{\boxhi} \subseteq \nnInputSet_0$ can be enclosed by
  \begin{equation*}
    f(\nnInputSet) \subseteq \shortI{\decompFun{f}(\boxlo,\boxhi)} {\decompFun{f}(\boxhi,\boxlo)}.
  \end{equation*}
\end{proposition}

Decomposition functions are not unique~\citep{coogan_2020}; hence, finding a tight decomposition is a central challenge.
In this work, we build neural networks to be mixed-monotone functions and we directly optimize the tightness of output enclosures through neural network training.

A standard decomposition function is the difference of two monotone functions~\citep[Ex.~1]{coogan_arcak_2015}, i.e., given the  monotone functions $\monoInc$ and $\monoDec$, $f(\nnInput) = \monoInc(\nnInput) - \monoDec(\nnInput)$, is a mixed-monotone function with decomposition function
\begin{equation}\label{eq:decomp_diff}
  \decompFun{f}(\nnInput_1,\nnInput_2) = \monoInc(\nnInput_1) - \monoDec(\nnInput_2)\text{.}
\end{equation}
We build our neural network architecture based on the decomposition in~\eqref{eq:decomp_diff}. Future work can explore more general forms of decomposition functions.

\subsection{Monotonic Neural networks}
Feed-forward networks~$\NN\colon \R^{n_0} \to \R^{n_\numLayers}$ are sequences of $\numLayers\in\N$ layers.
Each layer is either affine (linear layer) or an element-wise nonlinear activation function (nonlinear layer).
In this work, we consider sign-constrained neural networks that can only represent monotonic functions.
\begin{definition}[Sign-Constrained Neural Network, {\citep[Def.~1]{runje_2023}}]\label{def:neural_network}
  For an input $\nnInput\in\R^{\numNeurons_0}$, the output of a sign-constrained neural network $\nnOutput = \NN(\nnInput)\in\R^{\numNeurons_\numLayers}$ is
  \begin{align*}
    \nnHidden_0 & = \nnInput\text{,}                                                &
    \nnHidden_k & = \nnLayer{k}{\nnHidden_{k-1}}\quad\text{for $k\in[\numLayers]$,} &
    \nnOutput   & = \nnHidden_\numLayers\text{,}
  \end{align*}
  where
  \begin{equation*}
    \nnLayer{k}{\nnHidden_{k-1}} = \begin{cases}
      \abs{\Weights_k}\,\nnHidden_{k-1} + \bias_k & \text{if the $k$-th layer is linear,} \\
      \actfun_k(\nnHidden_{k-1})                  & \text{otherwise,}
    \end{cases}
  \end{equation*}
  with weights $\Weights_k\in\R^{\numNeurons_k\times \numNeurons_{k-1}}$, bias $\bias_k\in\R^{\numNeurons_k}$, and the (element-wise) nonlinear function $\actfun_k$.
\end{definition}
Any sign-constrained neural network represents a monotonic function~\citep[Lem.~1]{runje_2023}.
For the gradient computation during training, the weights remain unconstrained~(\Cref{app:nonnegative-weights}).
The activation functions can be used to determine whether the function represented by a monotonic neural network is convex or concave~\citep[Cor.~3]{runje_2023}.

\subsection{Problem Statement}
Given a neural network $\NN\colon \R^{n_0} \to \R^{n_\numLayers}$, an input set $\nnInputSet \subset \R^{n_0}$, and an unsafe set $\mathcal{U} \subset \R^{n_\numLayers}$, the goal of neural network verification is to prove that all inputs are safe, i.e., $\forall \nnInput \in \nnInputSet\colon \NN(\nnInput) \notin \mathcal{U}$.
To this end, we compute an enclosure $\mathcal{Y} \supseteq \{\NN(\nnInput) \mid \nnInput \in \nnInputSet\}$ of the output set and check whether $\mathcal{Y} \cap \mathcal{U} = \emptyset$.
We want to build a neural network architecture that allows an efficient computation of tight output enclosures to make neural network verification practical and scalable.

\section{A monotone-difference architecture}\label{sec:arch}
We propose a neural network architecture that realizes mixed-monotone functions using the difference of two monotone neural subnetworks:
Let $\monoNNInc$, $\monoNNDec\colon \R^{\numNeurons_0} \to \R^{\numNeurons_\numLayers}$ be monotonic neural networks, we construct a neural network~$\NN$ as
\begin{equation}\label{eq:arch}
  \NN = \monoNNInc - \monoNNDec\text{.}
\end{equation}
We use sign-constraints to realize the monotone neural subnetworks~$\monoNNInc$, $\monoNNDec$~(\cref{def:neural_network}).
The neural network~$\NN$ is mixed-monotone by construction and we can use the decomposition function to enclose the image of the neural network.
\begin{proposition}\label{prop:enclosure}
  Given an input set~$\nnInputSet=\shortI{\boxlo}{\boxhi}\subseteq\R^{\numNeurons_0}$, a neural network~$\NN = \monoNNInc - \monoNNDec$, which is the difference of two monotonic neural networks~$\monoNNInc,\monoNNDec$, we can enclose the output set of the neural network by
\begin{equation*}\label{eq:enclosure}
  \opEnclose{\NN}{\nnInputSet}
    = \shortI{\decompFun{\NN}(\boxlo,\boxhi)}
              {\decompFun{\NN}(\boxhi,\boxlo)}
    \supseteq \NN(\nnInputSet)\text{,}
\end{equation*}
  where
  \begin{equation*}\label{eq:decomp}
    \decompFun{\NN}(\nnInput_1,\nnInput_2) = \monoNNInc(\nnInput_1) - \monoNNDec(\nnInput_2)\text{.}
  \end{equation*}
  \begin{proof}
    The neural network~$\NN$ is mixed monotone~\citep[Ex.~1]{coogan_arcak_2015}, hence we can enclose the output using the decomposition function~(\cref{prop:soundness}).
  \end{proof}
\end{proposition}
The computation of the enclosure only requires two forward passes through the neural network via the decomposition function, instead of requiring a complicated set propagation.

\paragraph{Computational Complexity} \Cref{prop:enclosure} computes an enclosure with two forward passes, so that computing it is as expensive as a single point evaluation up to a constant factor, i.e., it takes $\bigO(\numNeurons_{\max}^2\,\numLayers)$ time, where $\numNeurons_{\max}=\max_{k\in\{0,\dots,\numLayers\}}\numNeurons_k$ denotes the largest number of neurons of a layer, including the input layer.
In contrast, propagating a zonotope through the same neural network takes $\bigO(\numNeurons_{\max}^2\,\numGens\,\numLayers)$ time, where the number of generators grows with the size of the neural network, i.e., $\numGens\in \bigO(\numNeurons_0 + \sum_{k\in[\numLayers]}\numNeurons_k)$~\citep[Prop.~16]{koller_et_al_2025_tmlr}.
Linear relaxations, e.g., CROWN~\citep{zhang_2018} or DeepPoly~\citep{singh_et_al_2019deeppoly}, take $\bigO(\numNeurons_{\max}^3\,\numLayers^2)$ time~\citep[Sec.~3.1]{zhang_2018}, as the bounds of each neuron are computed by a backward pass through all previous layers in $\bigO(\numNeurons_{\max}^2\,\numLayers)$ time~\citep[Sec.~2]{singh_et_al_2019deeppoly}.
Compared to both, our enclosure saves a factor that grows linearly with the number of neurons of the neural network.

\section{Verification-aware training}\label{sec:training}\label{sec:loss}
\begin{figure}[t]
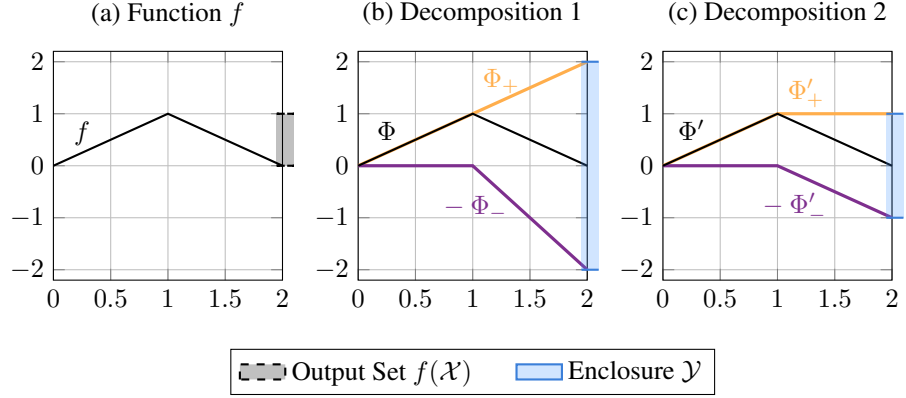

  \centering
  \includetikz{decomposition}
  \caption{The size of an enclosure depends on the decomposition: For the same function~$f$, two decompositions yield enclosures of different sizes.}
  \label{fig:decomposition}
\end{figure}
The decomposition of a mixed-monotone function is not unique, and different decompositions of the same function yield enclosures of different size.
\cref{fig:decomposition} compares two decompositions, which yield enclosures of different sizes.
First, we use an example to demonstrate the non-uniqueness of a decomposition and its effects on the size of the enclosure, before we explain how we use neural network training to optimize the tightness of an enclosure.
\begin{example}[Comparing Decompositions]\label{ex:decomposition}
  Let $\nnInputSet = \shortI{0}{2}$ be the input interval and let $f(\nnInput) = 1 - \abs{\nnInput - 1}$ be a function~(see \cref{fig:decomposition}(a)).
  The exact image of~$f$ (which generally is not computable) has size 1, i.e., $\size{f(\nnInputSet)} = \size{\shortI{0}{1}} = 1$.

  The function~$f$ can be represented by the mixed-monotone neural network $\NN = \monoNNInc - \monoNNDec$, with monotone subnetworks~(see~\cref{fig:decomposition}(b))
  \begin{align}\label{eq:decomp1}
    \monoNNInc(\nnInput) & = \nnInput\text{,}                &
    \monoNNDec(\nnInput) & = 2\,\max(\nnInput - 1,0)\text{.}
  \end{align}
  Using~\cref{prop:soundness}, we can enclose the output set of~$\NN$ by
  \begin{equation*}
    \opEnclose{\NN}{\nnInputSet} = \shortI{\monoNNInc(0) - \monoNNDec(2)}{\monoNNInc(2) - \monoNNDec(0)} = \shortI{-2}{2}\text{.}
  \end{equation*}
  The size of the enclosure is $\size{\opEnclose{\NN}{\nnInputSet}} = 4$.

  In contrast, another mixed-monotone neural network $\NN' = \monoNNInc['] - \monoNNDec[']$ can also represent the function $f$ with monotone subnetworks~(see~\cref{fig:decomposition}(c))
  \begin{align}\label{eq:decomp2}
    \monoNNInc['](\nnInput) & = \min(\nnInput,1)\text{,}                &
    \monoNNDec['](\nnInput) & = \max(\nnInput - 1,0)\text{.}
  \end{align}
  The computed enclosure is smaller:
  \begin{equation*}
    \opEnclose{\NN'}{\nnInputSet} = \shortI{\monoNNInc['](0) - \monoNNDec['](2)}{\monoNNInc['](2) - \monoNNDec['](0)} = \shortI{-1}{1}\text{,}
  \end{equation*}
  which has size $\size{\opEnclose{\NN'}{\nnInputSet}} = 2$.
\end{example}
\paragraph{Training Loss} Our goal is to directly optimize the tightness of the enclosure during training.
To this end, we combine a standard regression loss (i.e., mean-squared error) with a penalty of the enclosure size for a fixed input perturbation radius~$\epsilon\in\R^{\numNeurons_0}$.
For an input~$\nnInput_j\in\R^{\numNeurons_0}$ with target~$\nnTarget_j\in\R^{\numNeurons_\numLayers}$, we compute the training loss
\begin{equation}\label{eq:loss}
  \pointLoss
  = \underbrace{\norm{\NN(\nnInput_j) - \nnTarget_j}_2^2}_{\text{Regression Loss}} + \widthWeight \underbrace{\size{\opEnclose{\NN}{\shortI{\nnInput_j-\epsilon}{\nnInput_j+\epsilon}}}}_{\text{Enclosure Size}}\text{.}
\end{equation}
In principle, the training loss is similar to set-based training~\citep{koller_et_al_2025_tmlr,wendl_et_al_2026_tmlr}: The training objective penalizes the size of an output enclosure. However, we do not require set propagations during training, which makes the training significantly more efficient.

Evaluating the training loss~\eqref{eq:loss} only requires three forward passes, i.e., one for the regression loss and two for the enclosure~(\cref{prop:enclosure}); hence, a training iteration takes $\bigO(\numNeurons_{\max}^2\,\numLayers)$ time, which is as expensive as standard point-based training up to a constant factor.
In contrast, a set-based training iteration takes $\bigO(\numNeurons_{\max}^2\,\numGens\,\numLayers)$ time~\citep[Sec.~4.4]{koller_et_al_2025_tmlr}. Hence, our training saves the same factor~$\numGens$ as our enclosure computation.

Subsequently, we derive a lower bound of the enclosure size, which motivates the training objective and serves as a reference to measure the tightness of computed enclosures.

\paragraph{Enclosure Size}
The size of an enclosure cannot be smaller than the total variation of the output of the neural network along chains in the input interval, which generalizes the total variation of a scalar function~\citep[Sec.~4.4]{cohn1993measure} to several inputs~\citep[Sec.~2, Def.~A]{Clarkson1933OnDO}.
A chain in an interval $\nnInputSet = \shortI{\boxlo}{\boxhi} \subseteq \R^{\numNeurons_0}$ is a tuple $(\nnInput_0,\dots,\nnInput_p)$ of finitely many inputs with $\boxlo \leq \nnInput_0 \leq \nnInput_1 \leq \dots \leq \nnInput_p \leq \boxhi$ with respect to the component-wise partial order~(\cref{sec:prelim}), where the length~$p\in\N$ is arbitrary.
The total variation of the output of a neural network~$\NN$ along chains in~$\nnInputSet$ is the vector~$\monVar{\NN}{\nnInputSet}\in\R^{\numNeurons_\numLayers}$, whose $i$-th entry, for $i\in[\numNeurons_\numLayers]$, is the largest summed change of the $i$-th output along a chain:
\begin{equation}\label{eq:variation}
  \monVar{\NN}{\nnInputSet}_{(i)} = \sup\,\biggl\{\,\sum_{\ell=1}^{p} \abs{\NN(\nnInput_\ell)_{(i)} - \NN(\nnInput_{\ell-1})_{(i)}} \;\biggm|\; p\in\N\text{, } \boxlo \leq \nnInput_0 \leq \nnInput_1 \leq \dots \leq \nnInput_p \leq \boxhi \biggr\}\text{.}
\end{equation}
Between two consecutive inputs of a chain, the output of~$\NN$ changes by at most the sum of the increases of~$\monoNNInc$ and~$\monoNNDec$, as neither subnetwork can decrease.
\newcommand{\lowerBoundProp}{
  Given an input interval $\nnInputSet = \shortI{\boxlo}{\boxhi} \subseteq \R^{\numNeurons_0}$ and a neural network $\NN = \monoNNInc - \monoNNDec$, where $\monoNNInc,\monoNNDec$ are non-decreasing on~$\nnInputSet$~(\cref{def:monotone}), it holds that
  \begin{equation*}
    \ones^\top\,\monVar{\NN}{\nnInputSet}\leq\size{\opEnclose{\NN}{\nnInputSet}}\text{.}
  \end{equation*}
  Moreover, there exist non-decreasing functions $\monoNNInc['],\monoNNDec[']\colon\nnInputSet\to\R^{\numNeurons_\numLayers}$ with $\monoNNInc['] - \monoNNDec['] = \NN$ on~$\nnInputSet$, such that equality holds for $\NN' = \monoNNInc['] - \monoNNDec[']$.
}
\begin{proposition}[Smallest enclosure]\label{prop:lowerbound}
  \lowerBoundProp
  \begin{proof}
    See \Cref{app:proofs}.
  \end{proof}
\end{proposition}
The lower bound~$\ones^\top\monVar{\NN}{\nnInputSet}$ only depends on the function represented by the neural network and not on its decomposition.

\section{Evaluation}
\label{sec:eval}

We have implemented our mixed-monotone neural network architecture in MATLAB using the CORA toolbox~\citep{althoff_2015}.
Our evaluation is structured as follows: We first assess prediction accuracy of mixed-monotone neural networks, next compare enclosure tightness and computational cost, and finally showcase the effectiveness of verification-aware training.
Finally, we demonstrate the superior scalability of our approach with respect to the input dimension and number of neural network layers and additionally evaluate the enclosure tightness over multiple time steps of a dynamical system.

\paragraph{Benchmarks} We evaluate our approach on two static analytic models, i.e., a scalar analytic model and a multi-input analytic model with a variable number of inputs, and on two dynamical systems (Pendulum and Tanks).
The scalar analytic model has a closed-form total variation, which serves as an exact reference for the lower bound of~\cref{prop:lowerbound}.
The two dynamical systems cover a cooperative system (Tanks), i.e., a system whose solutions preserve the partial order of the initial states for a constant input~\citep{smith_1995}, and a non-cooperative system (Pendulum).
In the multi-input analytic model, only two inputs enter nonlinearly, so that we can increase the input dimension without changing the nonlinearity of the model.

\paragraph{Metrics} We report prediction accuracy using $R^2$-scores.
To measure the tightness of a computed enclosure, we report the ratio of its size to the size of the exact image, i.e., for an input set~$\nnInputSet$ and enclosure~$\nnOutputSet \supseteq \NN(\nnInputSet)$, we report~$\size{\nnOutputSet}/\size{\NN(\nnInputSet)}$, which we refer to as the tightness of the enclosure.
Since computing the exact image is generally intractable, we underapproximate its size using samples.
Since $\size{\NN(\nnInputSet)} \leq \ones^\top\,\monVar{\NN}{\nnInputSet}$ by \cref{prop:soundness,prop:lowerbound}, the tightness is also an upper bound of the ratio~$\size{\nnOutputSet}/(\ones^\top\monVar{\NN}{\nnInputSet})$ of the enclosure size to the lower bound.

\subsection{Prediction accuracy}\label{sec:accuracy}
Our sign-constrained architecture maintains high prediction accuracy across the tested benchmarks and the number of neural network layers, showing that the monotonicity constraints do not prevent effective training.
\Cref{tab:accuracy-main} compares our architecture with an unconstrained baseline under the same training protocol, with differences in mean $R^2$ of at most $0.011$ across the reported configurations.
Additional results for the scalar analytic model are provided in~\Cref{tab:accuracy}.
\begin{table}
  \centering
  \caption{Comparing prediction accuracy across different benchmarks shown as mean $\pm$ standard deviation over ten seeds.}
  \label{tab:accuracy-main}
  \begin{tabular}{l c R@{$\,\pm\,$}L R@{$\,\pm\,$}L}
    \toprule
    \multirow{2}{*}{\textbf{Benchmark}} & \multirow{2}{*}{\textbf{\#Layers}~$\numLayers$} & \multicolumn{4}{c}{\textbf{Accuracy}~$\uparrow$}                                                       \\
    \cmidrule(lr){3-6}
                                        &                                                 & \multicolumn{2}{c}{Baseline}                     & \multicolumn{2}{c}{Ours}                            \\
    \midrule
    \multirow{3}{*}{Scalar analytic}
                                        & 2                                               & 0.9837                                           & 0.0018                   & \mathbf{0.9858} & 0.0009 \\
                                        & 4                                               & \mathbf{0.9998}                                  & 0.0001                   & 0.9895          & 0.0051 \\
                                        & 8                                               & \mathbf{0.9997}                                  & 0.0003                   & 0.9964          & 0.0035 \\
    \midrule
    \multirow{2}{*}{Pendulum}
                                        & 3                                               & \mathbf{0.9942}                                  & 0.0024                   & 0.9878          & 0.0071 \\
                                        & 8                                               & \mathbf{0.9931}                                  & 0.0049                   & 0.9868          & 0.0098 \\
    \midrule
    \multirow{2}{*}{Tanks}
                                        & 3                                               & \mathbf{0.9997}                                  & 0.0001                   & 0.9993          & 0.0004 \\
                                        & 8                                               & \mathbf{0.9985}                                  & 0.0013                   & 0.9984          & 0.0012 \\
    \bottomrule
  \end{tabular}
\end{table}

\subsection{Verification tightness and runtimes}
\label{sec:verifiers}
\label{sec:crossover}
In~\cref{tab:verifiers-main}, we compare the tightness and runtimes of our enclosure (Ours) with interval bound propagation~\citep[IBP;][]{gowal_2019}, zonotope propagation in CORA~\citep{althoff_2015}, CROWN~\citep{zhang_2018}, and $\alpha$-CROWN~\citep{xu_2021} for neural networks trained on the scalar analytic model; see \Cref{app:timing,app:lirpa} for details and additional results.
Since the methods are implemented in different toolboxes, \texttt{auto\_LiRPA} (PyTorch) and CORA (MATLAB), runtimes are reported per toolbox and relative to our enclosure within that toolbox to provide a fair comparison.
At four layers, the enclosures of all methods exceed the size of the sampled image by at most $17\,\%$, and our enclosure is as tight as zonotope propagation. At eight layers, our enclosure is tighter than zonotope propagation and CROWN, and its size exceeds that of $\alpha$-CROWN by only $5.1\,\%$, i.e., the tightness of our enclosure improves with depth while that of the set-propagation methods degrades.
Computing our enclosure is $32\times$ faster than zonotope propagation, $256\times$ faster than CROWN, and $20\,600\times$ faster than $\alpha$-CROWN, which additionally optimizes its bounds.
\begin{table}
  \centering
  \caption{Comparing enclosure tightness and runtimes on the scalar analytic model shown as mean $\pm$ standard deviation over five seeds. The runtime is the time to enclose the output set for one input interval, reported as the minimum over repeated measurements~(\cref{app:timing,app:lirpa}).
    We evaluate our approach within CORA as well as within \texttt{auto\_LiRPA}.
  }
  \label{tab:verifiers-main}
  \begin{tabular}{c l l R@{$\,\pm\,$}L R R}
    \toprule
    \multirow{2}{*}{\textbf{\#Layers}~$\numLayers$} & \multirow{2}{*}{\textbf{Framework}}   & \multirow{2}{*}{\textbf{Method}} & \multicolumn{2}{c}{\multirow{2}{*}{\textbf{Tightness}~$\downarrow$}} & \multicolumn{2}{c}{\textbf{Runtime}~$\downarrow$}                                               \\
    \cmidrule(lr){6-7}
                                                    &                                       &                                  & \multicolumn{2}{c}{}                                                 & \textit{[\textnormal{ms}]}                        & \textit{[relative to ours]}                 \\
    \midrule
    \multirow{6}{*}{4}
                                                    & \multirow{4}{*}{\texttt{auto\_LiRPA}} & Ours (exported)                  & 1.1662                                                               & 0.1167                                            & \mathbf{0.12}               & 1\times       \\
                                                    &                                       & IBP                              & 1.1662                                                               & 0.1167                                            & 1.21                        & 9.9\times     \\
                                                    &                                       & CROWN                            & 1.1264                                                               & 0.0371                                            & 13.7                        & 112\times     \\
                                                    &                                       & $\alpha$-CROWN                   & \mathbf{1.0024}                                                      & 0.0025                                            & 1\,070                      & 8\,750\times  \\
    \cmidrule(lr){2-7}
                                                    & \multirow{2}{*}{CORA}                 & Ours                             & 1.1662                                                               & 0.1167                                            & \mathbf{0.01}               & 1\times       \\
                                                    &                                       & Zonotope                         & 1.1592                                                               & 0.0483                                            & 0.17                        & 11.7\times    \\
    \midrule
    \multirow{6}{*}{8}
                                                    & \multirow{4}{*}{\texttt{auto\_LiRPA}} & Ours (exported)                  & 1.0534                                                               & 0.0805                                            & \mathbf{0.21}               & 1\times       \\
                                                    &                                       & IBP                              & 1.0534                                                               & 0.0805                                            & 2.35                        & 11.2\times    \\
                                                    &                                       & CROWN                            & 1.2640                                                               & 0.0608                                            & 53.8                        & 256\times     \\
                                                    &                                       & $\alpha$-CROWN                   & \mathbf{1.0022}                                                      & 0.0012                                            & 4\,330                      & 20\,600\times \\
    \cmidrule(lr){2-7}
                                                    & \multirow{2}{*}{CORA}                 & Ours                             & 1.0534                                                               & 0.0805                                            & \mathbf{0.02}               & 1\times       \\
                                                    &                                       & Zonotope                         & 1.4409                                                               & 0.0655                                            & 0.69                        & 32\times      \\
    \bottomrule
  \end{tabular}
\end{table}
Due to the monotone subnetworks, the IBP bounds coincide with ours, but our bounds are computed using two simple forward passes without any set propagation, which makes it $10\times$ to $11\times$ faster than IBP in \texttt{auto\_LiRPA}.

\subsection{Verification-Aware Training and the lower bound}\label{sec:betasweep}
To evaluate the effectiveness of our verification-aware training, we vary the parameter~$\widthWeight$ to weight the width of the output set in the training loss~(see \cref{eq:loss}) and compare the accuracy and enclosure tightness.
\cref{tab:beta} shows the results.
Using~$\widthWeight$, we can tradeoff accuracy against enclosure tightness: A small value increases accuracy, while a larger value favors enclosure tightness.
\begin{table}
  \centering
  \caption{Comparing the accuracy and enclosure tightness for varying~$\widthWeight$ in the loss function~\cref{eq:loss} shown as mean $\pm$ standard deviation over three seeds.}
  \label{tab:beta}
  \begin{tabular}{l R R@{$\,\pm\,$}L R@{$\,\pm\,$}L}
    \toprule
    \textbf{Benchmark} & $\widthWeight$ & \multicolumn{2}{c}{\textbf{Accuracy}~$\uparrow$} & \multicolumn{2}{c}{\textbf{Tightness}~$\downarrow$}                          \\
    \midrule
    \multirow{5}{*}{Scalar analytic}
                       & 0              & 0.9846                                           & 0.0007                                              & 1.821          & 0.136 \\
                       & 0.01           & 0.9857                                           & 0.0034                                              & 1.024          & 0.009 \\
                       & 0.05           & \mathbf{0.9883}                                  & 0.0012                                              & 1.004          & 0.003 \\
                       & 0.2            & 0.9803                                           & 0.0012                                              & \mathbf{1.000} & 0.000 \\
                       & 1.0            & 0.9392                                           & 0.0121                                              & \mathbf{1.000} & 0.000 \\
    \midrule
    \multirow{5}{*}{Tanks}
                       & 0              & \mathbf{0.9999}                                  & 0.0000                                              & 1.883          & 0.134 \\
                       & 0.01           & 0.9998                                           & 0.0001                                              & \mathbf{1.000} & 0.000 \\
                       & 0.05           & 0.9985                                           & 0.0006                                              & \mathbf{1.000} & 0.000 \\
                       & 0.2            & 0.9852                                           & 0.0035                                              & \mathbf{1.000} & 0.000 \\
                       & 1.0            & 0.8646                                           & 0.0121                                              & \mathbf{1.000} & 0.000 \\
    \bottomrule
  \end{tabular}
\end{table}

For $\widthWeight \geq 0.05$, the enclosure coincides with the sampled image on both benchmarks, i.e., the decomposition is already tight; a larger $\widthWeight$ can then only reduce the enclosure size by reducing the variation of the function represented by the neural network, which decreases its accuracy.

\subsection{Scalability}\label{sec:dims}

Due to mixed-monotonicity of the neural network, our enclosure computation does not require to explicitly consider approximation errors.
Set propagation approaches suffer from the wrapping effect of approximation errors across layers, which prevents the scalability of these approaches.
We demonstrate the superior scalability of our approach across different dimensions:
\begin{enumerate*}[label=(\roman*)]
  \item the number of input dimensions,
  \item the number of neural network layers, and
  \item the number of time steps over which the neural network is applied iteratively.
\end{enumerate*}

\paragraph{Number of Input Dimensions}
We significantly increase the number of input dimensions of the multi-input analytic model and compare the enclosure tightness of our approach against the baseline using zonotope propagation.
\cref{tab:dims-main} shows the results: While the tightness of the baseline explodes for $n=100$, our approach remains significantly tighter.
Further results are provided in~\Cref{tab:dims}.
\begin{table}
  \centering
  \caption{Comparing the enclosure tightness for different numbers of input dimensions shown as mean $\pm$ standard deviation over five seeds.}
  \label{tab:dims-main}
  \begin{tabular}{c R@{$\,\pm\,$}L R@{$\,\pm\,$}L}
    \toprule
    \multirow{2}{*}{\textbf{\#Input Dims}} & \multicolumn{4}{c}{\textbf{Tightness}~$\downarrow$}                                                     \\
    \cmidrule(lr){2-5}
                                           & \multicolumn{2}{c}{Zonotope}                        & \multicolumn{2}{c}{Ours}                          \\
    \midrule
    2                                      & 1.65                                                & 0.01                     & \mathbf{1.045} & 0.008 \\
    20                                     & 1.77                                                & 0.09                     & \mathbf{1.073} & 0.009 \\
    100                                    & 360                                                 & 31                       & \mathbf{1.167} & 0.039 \\
    \bottomrule
  \end{tabular}
\end{table}

\paragraph{Number of Layers}

We compare in~\cref{tab:crossover} the tightness of our approach against the baseline using zonotope propagation for different numbers of neural network layers at a fixed perturbation radius~$\epsilon = 0.05$.
Additionally, we report the smallest input set radius~(crossover radius) at which our enclosure is smaller than the corresponding zonotope enclosure.
When increasing the number of layers, the crossover radius becomes smaller. Additional results are in \Cref{app:crossover-results}.
\begin{table}
  \centering
  \caption{Comparing the enclosure tightness for different numbers of neural network layers shown as mean $\pm$ standard deviation over ten seeds. The crossover radius is reported as median (range) over seeds; ``--'' indicates no crossover.}
  \label{tab:crossover}
  \begin{tabular}{c R@{$\,\pm\,$}L R@{$\,\pm\,$}L c}
    \toprule
    \multirow{2}{*}{\textbf{\#Layers}~$\numLayers$} & \multicolumn{4}{c}{\textbf{Tightness}~$\downarrow$} & \multirow{2}{*}{\makecell{\textbf{Crossover Radius}~$\downarrow$                                                   \\ \textit{[median (range)]}}}\\
    \cmidrule(lr){2-5}
                                                    & \multicolumn{2}{c}{Zonotope}                        & \multicolumn{2}{c}{Ours}                                         &                                                 \\
    \midrule
    2                                               & \mathbf{1.08}                                       & 0.01                                                             & 5.92          & 0.79 & --                       \\
    3                                               & \mathbf{1.11}                                       & 0.02                                                             & 2.47          & 0.72 & $1.00$ ($0.1$--$1.5$)    \\
    4                                               & \mathbf{1.14}                                       & 0.03                                                             & 1.56          & 0.37 & $0.20$ ($0.03$--$0.5$)   \\
    6                                               & 1.24                                                & 0.12                                                             & \mathbf{1.13} & 0.12 & $0.025$ ($0.01$--$0.15$) \\
    8                                               & 1.38                                                & 0.25                                                             & \mathbf{1.19} & 0.18 & $0.04$ ($0.01$--$0.075$) \\
    \bottomrule
  \end{tabular}
\end{table}

\paragraph{Multiple time steps}\label{sec:closedloop}
Finally, we apply the neural network iteratively: we train a mixed-monotone neural network to predict the state of the Tanks system after one time step for a constant input and enclose the reachable set over several time steps by applying~\cref{prop:enclosure} repeatedly, where the enclosure of one time step is the input interval of the next~(\cref{app:closedloop-details}).
We compare against zonotope propagation through the unconstrained baseline trained on the same data.
\cref{fig:closedloop} shows the size of the enclosures over time: our enclosure stays close to the sampled reachable set, while the zonotope enclosure of the baseline diverges for $\numLayers=8$ layers.
\begin{figure}[tb]
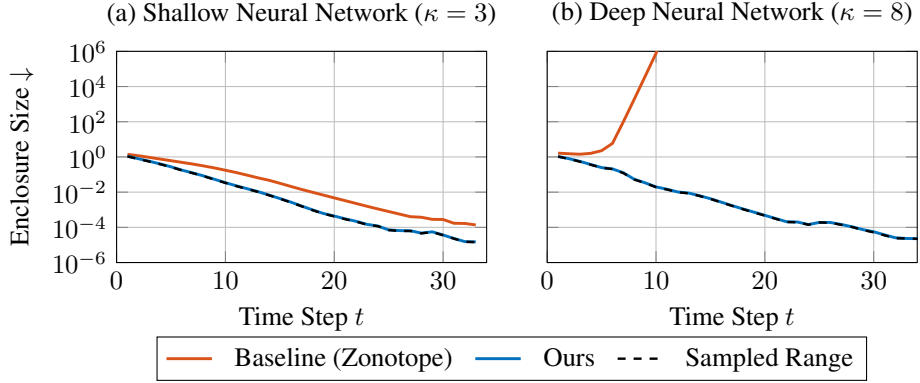

  \centering
  \includetikz{closedloop}
  \caption{Size of the enclosure of the reachable set over time for the Tanks system, when the neural network is applied iteratively.}
  \label{fig:closedloop}
\end{figure}

\section{Conclusion}
We introduce a neural network architecture that realizes mixed-monotone functions as the difference of two monotone subnetworks.
This architecture replaces set propagation by two forward passes of opposite corners of the input interval, which makes computing an output enclosure as cheap as a point evaluation.
Because the size of this enclosure is a closed-form function of the weights, we can directly minimize it during training, which turns the classical problem of finding a tight mixed-monotone decomposition into a learning problem; the total variation of the represented function along chains provides a lower bound with which a learned decomposition can be compared.
Our evaluation shows that (i) training reduces the enclosure size to within $0.5\,\%$ of this lower bound without sacrificing prediction accuracy, (ii) computing our enclosure is orders of magnitude faster than CROWN and $\alpha$-CROWN, and (iii) on deep neural networks the enclosure is tighter than zonotope propagation and CROWN, while its size exceeds that of $\alpha$-CROWN by only $5.1\,\%$.
The advantage grows where verification is hardest: with the depth of the neural network, with the number of input dimensions, and when the neural network is applied iteratively over multiple time steps.

\paragraph{Limitations}
Our enclosure bounds the output of the neural network, not the output of the system that the neural network approximates.
Due to our sign-constrained architecture, interval bound propagation returns the same enclosure; our advantages are that two forward passes suffice, which is $10\times$ to $11\times$ faster in our experiments, and that the size of the enclosure is directly minimized during training.
We restrict the decomposition function to the difference of two monotone functions~\eqref{eq:decomp_diff}, to which the lower bound of~\cref{prop:lowerbound} applies; more general decomposition functions can yield tighter enclosures and are future work.
Finally, we only consider regression tasks, and our evaluation over multiple time steps is restricted to a cooperative system with a constant input.

\section{AI Use Statement}
In this work, we used generative AI tools for  
drafting sections, 
identifying and summarizing relevant literature, 
assisting during the implementation and in the writing of proofs,
editing to improve readability of the paper, and
creating or modifying TikZ figures or images.
We have reviewed all AI-assisted work to ensure its correctness. 
We have carefully checked the correctness of our AI-assisted implementation and our reported results are reproducible using the included supplementary material. 
We have checked every claim and reference against its source.

\section{Reproducibility Statement}
All benchmark definitions and evaluation details are included in~\Cref{sec:benchmarks,sec:metrics,app:details}.
Additionally, we include a reproducibility package as supplementary material with this submission, which contains all source code and allows the reader to reproduce all results contained in the paper.

\section*{Acknowledgements}

This work was partially supported by the project SPP-2422 (No. 500936349) funded by the German Research Foundation (DFG).
 
\bibliography{refs.bib}

\begin{thebibliography}{36}
\providecommand{\natexlab}[1]{#1}
\providecommand{\url}[1]{\texttt{#1}}
\expandafter\ifx\csname urlstyle\endcsname\relax
  \providecommand{\doi}[1]{doi: #1}\else
  \providecommand{\doi}{doi: \begingroup \urlstyle{rm}\Url}\fi

\bibitem[Althoff(2015)]{althoff_2015}
Matthias Althoff.
\newblock An introduction to {CORA} 2015.
\newblock In \emph{Proc. of the Workshop on Applied Verification of Continuous
  and Hybrid Systems (ARCH)}, volume~34 of \emph{EPiC Series in Computing},
  pp.\  120--151, 2015.

\bibitem[Althoff(2023)]{althoff_2023}
Matthias Althoff.
\newblock Checking and establishing reachset conformance in {CORA} 2023.
\newblock In \emph{Proc. of the Workshop on Applied Verification of Continuous
  and Hybrid Systems (ARCH)}, volume~96 of \emph{EPiC Series in Computing},
  pp.\  9--33, 2023.

\bibitem[Bunel et~al.(2020)Bunel, Turkaslan, Torr, Kumar, Lu, and
  Kohli]{bunel_et_al_2020}
Rudy Bunel, Ilker Turkaslan, Philip H.~S. Torr, M.~Pawan Kumar, Jingyue Lu, and
  Pushmeet Kohli.
\newblock Branch and bound for piecewise linear neural network verification.
\newblock \emph{Journal of Machine Learning Research}, 21\penalty0
  (42):\penalty0 1--39, 2020.

\bibitem[Clarkson \& Adams(1933)Clarkson and Adams]{Clarkson1933OnDO}
James~A. Clarkson and Clifford~R. Adams.
\newblock On definitions of bounded variation for functions of two variables.
\newblock \emph{Transactions of the American Mathematical Society},
  35:\penalty0 824--854, 1933.

\bibitem[Cohn(1993)]{cohn1993measure}
Donald~L. Cohn.
\newblock \emph{Measure Theory}.
\newblock Birkh{\"a}user, Boston, 1993.

\bibitem[Coogan(2020)]{coogan_2020}
Samuel Coogan.
\newblock Mixed monotonicity for reachability and safety in dynamical systems.
\newblock In \emph{IEEE Conference on Decision and Control (CDC)}, pp.\
  5074--5085, 2020.

\bibitem[Coogan \& Arcak(2015)Coogan and Arcak]{coogan_arcak_2015}
Samuel Coogan and Murat Arcak.
\newblock Efficient finite abstraction of mixed monotone systems.
\newblock In \emph{Hybrid Systems: Computation and Control (HSCC)}, pp.\
  58--67, 2015.

\bibitem[Daniels \& Velikova(2010)Daniels and Velikova]{daniels_velikova_2010}
Hennie Daniels and Marina Velikova.
\newblock Monotone and partially monotone neural networks.
\newblock \emph{IEEE Transactions on Neural Networks}, 21\penalty0
  (6):\penalty0 906--917, 2010.

\bibitem[Girard(2005)]{girard_2005}
Antoine Girard.
\newblock Reachability of uncertain linear systems using zonotopes.
\newblock In \emph{Hybrid Systems: Computation and Control (HSCC)}, pp.\
  291--305, 2005.

\bibitem[Gowal et~al.(2019)Gowal, Dvijotham, Stanforth, Bunel, Qin, Uesato,
  Arandjelovic, Mann, and Kohli]{gowal_2019}
Sven Gowal, Krishnamurthy Dvijotham, Robert Stanforth, Rudy Bunel, Chongli Qin,
  Jonathan Uesato, Relja Arandjelovic, Timothy~A. Mann, and Pushmeet Kohli.
\newblock Scalable verified training for provably robust image classification.
\newblock In \emph{IEEE/CVF International Conference on Computer Vision
  (ICCV)}, pp.\  4841--4850, 2019.

\bibitem[Jafarpour et~al.(2023)Jafarpour, Harapanahalli, and
  Coogan]{jafarpour_2023}
Saber Jafarpour, Akash Harapanahalli, and Samuel Coogan.
\newblock Interval reachability of nonlinear dynamical systems with neural
  network controllers.
\newblock In \emph{Learning for Dynamics and Control Conference (L4DC)}, volume
  211 of \emph{Proceedings of Machine Learning Research}, pp.\  12--25, 2023.

\bibitem[Katz et~al.(2017)Katz, Barrett, Dill, Julian, and
  Kochenderfer]{katz_2017}
Guy Katz, Clark Barrett, David~L. Dill, Kyle Julian, and Mykel~J. Kochenderfer.
\newblock {R}eluplex: An efficient {SMT} solver for verifying deep neural
  networks.
\newblock In \emph{Computer Aided Verification (CAV)}, pp.\  97--117, 2017.

\bibitem[Khajenejad \& Yong(2023)Khajenejad and Yong]{khajenejad_2023}
Mohammad Khajenejad and Sze~Zheng Yong.
\newblock Tight remainder-form decomposition functions with applications to
  constrained reachability and guaranteed state estimation.
\newblock \emph{IEEE Transactions on Automatic Control}, 68\penalty0
  (12):\penalty0 7057--7072, 2023.

\bibitem[Kitouni et~al.(2023)Kitouni, Nolte, and Williams]{kitouni_2023}
Ouail Kitouni, Niklas Nolte, and Mike Williams.
\newblock Robust and provably monotonic networks.
\newblock \emph{Machine Learning: Science and Technology}, 4\penalty0
  (3):\penalty0 035020, 2023.

\bibitem[Kochdumper \& Althoff(2021)Kochdumper and
  Althoff]{kochdumper_althoff_2021}
Niklas Kochdumper and Matthias Althoff.
\newblock Sparse polynomial zonotopes: A novel set representation for
  reachability analysis.
\newblock \emph{IEEE Transactions on Automatic Control}, 66\penalty0
  (9):\penalty0 4043--4058, 2021.

\bibitem[Kochdumper et~al.(2023)Kochdumper, Schilling, Althoff, and
  Bak]{kochdumper_2023}
Niklas Kochdumper, Christian Schilling, Matthias Althoff, and Stanley Bak.
\newblock Open- and closed-loop neural network verification using polynomial
  zonotopes.
\newblock In \emph{NASA Formal Methods (NFM)}, pp.\  16--36, 2023.

\bibitem[Koller et~al.(2025)Koller, Ladner, and
  Althoff]{koller_et_al_2025_tmlr}
Lukas Koller, Tobias Ladner, and Matthias Althoff.
\newblock Set-based training for neural network verification.
\newblock \emph{Transactions on Machine Learning Research}, 2025.

\bibitem[Koller et~al.(2026)Koller, Ladner, and
  Althoff]{koller_et_al_2025_shadow}
Lukas Koller, Tobias Ladner, and Matthias Althoff.
\newblock Out of the shadows: Exploring a latent space for neural network
  verification.
\newblock In \emph{International Conference on Learning Representations
  (ICLR)}, 2026.

\bibitem[Ladner \& Althoff(2023)Ladner and Althoff]{ladner_althoff_2023}
Tobias Ladner and Matthias Althoff.
\newblock Automatic abstraction refinement in neural network verification using
  sensitivity analysis.
\newblock In \emph{Hybrid Systems: Computation and Control (HSCC)}, 2023.

\bibitem[Liu et~al.(2020)Liu, Han, Zhang, and Liu]{liu_2020}
Xingchao Liu, Xing Han, Na~Zhang, and Qiang Liu.
\newblock Certified monotonic neural networks.
\newblock In \emph{Advances in Neural Information Processing Systems
  (NeurIPS)}, volume~33, 2020.

\bibitem[Meyer(2022)]{meyer_2022}
Pierre-Jean Meyer.
\newblock Reachability analysis of neural networks using mixed monotonicity.
\newblock \emph{IEEE Control Systems Letters}, 6:\penalty0 3068--3073, 2022.

\bibitem[M{\"u}ller et~al.(2023)M{\"u}ller, Eckert, Fischer, and
  Vechev]{mueller_et_al_2023_sabr}
Mark~Niklas M{\"u}ller, Franziska Eckert, Marc Fischer, and Martin Vechev.
\newblock Certified training: Small boxes are all you need.
\newblock In \emph{International Conference on Learning Representations
  (ICLR)}, 2023.

\bibitem[Runje \& Shankaranarayana(2023)Runje and Shankaranarayana]{runje_2023}
Davor Runje and Sharath~M. Shankaranarayana.
\newblock Constrained monotonic neural networks.
\newblock In \emph{International Conference on Machine Learning (ICML)}, pp.\
  29338--29353, 2023.

\bibitem[Sill(1997)]{sill_1997}
Joseph Sill.
\newblock Monotonic networks.
\newblock In \emph{Advances in Neural Information Processing Systems
  (NeurIPS)}, volume~10, 1997.

\bibitem[Singh et~al.(2018)Singh, Gehr, Mirman, P{\"u}schel, and
  Vechev]{singh_2018}
Gagandeep Singh, Timon Gehr, Matthew Mirman, Markus P{\"u}schel, and Martin
  Vechev.
\newblock Fast and effective robustness certification.
\newblock In \emph{Advances in Neural Information Processing Systems
  (NeurIPS)}, volume~31, 2018.

\bibitem[Singh et~al.(2019)Singh, Gehr, P{\"u}schel, and
  Vechev]{singh_et_al_2019deeppoly}
Gagandeep Singh, Timon Gehr, Markus P{\"u}schel, and Martin Vechev.
\newblock An abstract domain for certifying neural networks.
\newblock \emph{Proceedings of the ACM on Programming Languages}, 3\penalty0
  (POPL):\penalty0 41:1--41:30, 2019.

\bibitem[Sivaraman et~al.(2020)Sivaraman, Farnadi, Millstein, and Van~den
  Broeck]{sivaraman_2020}
Aishwarya Sivaraman, Golnoosh Farnadi, Todd Millstein, and Guy Van~den Broeck.
\newblock Counterexample-guided learning of monotonic neural networks.
\newblock In \emph{Advances in Neural Information Processing Systems
  (NeurIPS)}, volume~33, 2020.

\bibitem[Smith(1995)]{smith_1995}
Hal~L. Smith.
\newblock \emph{Monotone Dynamical Systems: An Introduction to the Theory of
  Competitive and Cooperative Systems}, volume~41 of \emph{Mathematical Surveys
  and Monographs}.
\newblock American Mathematical Society, 1995.

\bibitem[Towers et~al.(2024)Towers, Kwiatkowski, Terry, Balis, De~Cola, Deleu,
  Goul{\~a}o, Kallinteris, Krimmel, KG, et~al.]{towers_2024}
Mark Towers, Ariel Kwiatkowski, Jordan Terry, John~U. Balis, Gianluca De~Cola,
  Tristan Deleu, Manuel Goul{\~a}o, Andreas Kallinteris, Markus Krimmel, Arjun
  KG, et~al.
\newblock Gymnasium: A standard interface for reinforcement learning
  environments.
\newblock arXiv:2407.17032, 2024.

\bibitem[Wendl et~al.(2026)Wendl, Koller, Ladner, and
  Althoff]{wendl_et_al_2026_tmlr}
Manuel Wendl, Lukas Koller, Tobias Ladner, and Matthias Althoff.
\newblock Training verifiably robust agents using set-based reinforcement
  learning.
\newblock \emph{Transactions on Machine Learning Research}, 2026.

\bibitem[Xu et~al.(2020)Xu, Shi, Zhang, Wang, Chang, Huang, Kailkhura, Lin, and
  Hsieh]{xu_2020}
Kaidi Xu, Zhouxing Shi, Huan Zhang, Yihan Wang, Kai-Wei Chang, Minlie Huang,
  Bhavya Kailkhura, Xue Lin, and Cho-Jui Hsieh.
\newblock Automatic perturbation analysis for scalable certified robustness and
  beyond.
\newblock In \emph{Advances in Neural Information Processing Systems
  (NeurIPS)}, volume~33, 2020.

\bibitem[Xu et~al.(2021)Xu, Zhang, Wang, Wang, Jana, Lin, and Hsieh]{xu_2021}
Kaidi Xu, Huan Zhang, Shiqi Wang, Yihan Wang, Suman Jana, Xue Lin, and Cho-Jui
  Hsieh.
\newblock Fast and complete: Enabling complete neural network verification with
  rapid and massively parallel incomplete verifiers.
\newblock In \emph{International Conference on Learning Representations
  (ICLR)}, 2021.

\bibitem[Yang et~al.(2019)Yang, Mickelin, and Ozay]{yang_2019}
Liren Yang, Oscar Mickelin, and Necmiye Ozay.
\newblock On sufficient conditions for mixed monotonicity.
\newblock \emph{IEEE Transactions on Automatic Control}, 64\penalty0
  (12):\penalty0 5080--5085, 2019.

\bibitem[Ye et~al.(2022)Ye, Wang, Wang, and Tie]{ye_2022}
Cunliang Ye, Yongfu Wang, Yunlong Wang, and Ming Tie.
\newblock Steering angle prediction {YOLOv5}-based end-to-end adaptive neural
  network control for autonomous vehicles.
\newblock \emph{Proceedings of the Institution of Mechanical Engineers, Part D:
  Journal of Automobile Engineering}, 236\penalty0 (9):\penalty0 1991--2011,
  2022.

\bibitem[You et~al.(2017)You, Ding, Canini, Pfeifer, and Gupta]{you_2017}
Seungil You, David Ding, Kevin Canini, Jan Pfeifer, and Maya Gupta.
\newblock Deep lattice networks and partial monotonic functions.
\newblock In \emph{Advances in Neural Information Processing Systems
  (NeurIPS)}, volume~30, 2017.

\bibitem[Zhang et~al.(2018)Zhang, Weng, Chen, Hsieh, and Daniel]{zhang_2018}
Huan Zhang, Tsui-Wei Weng, Pin-Yu Chen, Cho-Jui Hsieh, and Luca Daniel.
\newblock Efficient neural network robustness certification with general
  activation functions.
\newblock In \emph{Advances in Neural Information Processing Systems
  (NeurIPS)}, volume~31, 2018.

\end{thebibliography}
\bibliographystyle{iclr2027_conference}

\appendix

\section{Appendix}

\subsection{Proofs}
\label{app:proofs}

\begin{repproposition}{prop:lowerbound}
  \lowerBoundProp
  \begin{proof}
    We first compute the size of the enclosure~$\opEnclose{\NN}{\nnInputSet}$, whose lower and upper bounds are $\monoNNInc(\boxlo) - \monoNNDec(\boxhi)$ and $\monoNNInc(\boxhi) - \monoNNDec(\boxlo)$~(\cref{prop:enclosure}).
    Since $\boxlo \leq \boxhi$ and $\monoNNInc,\monoNNDec$ are non-decreasing on~$\nnInputSet$~(\cref{def:monotone}), the difference between the upper and the lower bound is
    \begin{equation*}
      \bigl(\monoNNInc(\boxhi) - \monoNNDec(\boxlo)\bigr) - \bigl(\monoNNInc(\boxlo) - \monoNNDec(\boxhi)\bigr)
      = \bigl(\monoNNInc(\boxhi) - \monoNNInc(\boxlo)\bigr) + \bigl(\monoNNDec(\boxhi) - \monoNNDec(\boxlo)\bigr) \geq \zeros\text{.}
    \end{equation*}
    Hence, $\opEnclose{\NN}{\nnInputSet}$ is a non-empty interval, and its interval norm~(\cref{sec:prelim}) is
    \begin{equation}\label{eq:enclosure-size}
      \size{\opEnclose{\NN}{\nnInputSet}}
      = \ones^\top\bigl(\monoNNInc(\boxhi) - \monoNNInc(\boxlo) + \monoNNDec(\boxhi) - \monoNNDec(\boxlo)\bigr)\text{.}
    \end{equation}

    Next, we show that $\ones^\top\monVar{\NN}{\nnInputSet} \leq \size{\opEnclose{\NN}{\nnInputSet}}$.
    Let $(\nnInput_0,\dots,\nnInput_p)$ be a chain in~$\nnInputSet$.
    For all $\ell\in[p]$, we have $\nnInput_{\ell-1} \leq \nnInput_\ell$, so that $\monoNNInc(\nnInput_\ell) - \monoNNInc(\nnInput_{\ell-1}) \geq \zeros$ and $\monoNNDec(\nnInput_\ell) - \monoNNDec(\nnInput_{\ell-1}) \geq \zeros$~(\cref{def:monotone}).
    With $\NN = \monoNNInc - \monoNNDec$~\eqref{eq:arch} and the triangle inequality, we obtain
    \begin{equation}\label{eq:step-bound}
      \abs{\NN(\nnInput_\ell) - \NN(\nnInput_{\ell-1})}
      \leq \bigl(\monoNNInc(\nnInput_\ell) - \monoNNInc(\nnInput_{\ell-1})\bigr) + \bigl(\monoNNDec(\nnInput_\ell) - \monoNNDec(\nnInput_{\ell-1})\bigr)\text{.}
    \end{equation}
    We sum~\eqref{eq:step-bound} over all $\ell\in[p]$.
    On the right-hand side, all intermediate terms cancel; for~$\monoNNInc$, we obtain
    \begin{equation*}
      \sum_{\ell=1}^{p} \bigl(\monoNNInc(\nnInput_\ell) - \monoNNInc(\nnInput_{\ell-1})\bigr)
      = \monoNNInc(\nnInput_p) - \monoNNInc(\nnInput_0)
      \leq \monoNNInc(\boxhi) - \monoNNInc(\boxlo)\text{,}
    \end{equation*}
    where the inequality holds because $\boxlo \leq \nnInput_0$, $\nnInput_p \leq \boxhi$, and $\monoNNInc$ is non-decreasing on~$\nnInputSet$; the same holds for~$\monoNNDec$.
    Therefore,
    \begin{equation}\label{eq:chain-bound}
      \sum_{\ell=1}^{p} \abs{\NN(\nnInput_\ell) - \NN(\nnInput_{\ell-1})}
      \leq \monoNNInc(\boxhi) - \monoNNInc(\boxlo) + \monoNNDec(\boxhi) - \monoNNDec(\boxlo)\text{.}
    \end{equation}
    The right-hand side of~\eqref{eq:chain-bound} does not depend on the chain; hence, it is an upper bound of the supremum in~\eqref{eq:variation}, i.e., $\monVar{\NN}{\nnInputSet} \leq \monoNNInc(\boxhi) - \monoNNInc(\boxlo) + \monoNNDec(\boxhi) - \monoNNDec(\boxlo)$.
    Multiplying both sides by~$\ones^\top$ and replacing the right-hand side by $\size{\opEnclose{\NN}{\nnInputSet}}$ using~\eqref{eq:enclosure-size} yields $\ones^\top\monVar{\NN}{\nnInputSet} \leq \size{\opEnclose{\NN}{\nnInputSet}}$.

    Finally, we show that the lower bound is attained by a suitable decomposition of~$\NN$.
    For any non-decreasing functions~$\monoNNInc[']$ and~$\monoNNDec[']$ with $\NN' = \monoNNInc['] - \monoNNDec[']$, rearranging~\eqref{eq:enclosure-size} gives
    \begin{equation}\label{eq:enclosure-size-sum}
      \size{\opEnclose{\NN'}{\nnInputSet}}
      = \ones^\top\Bigl(\bigl(\monoNNInc['](\boxhi) + \monoNNDec['](\boxhi)\bigr) - \bigl(\monoNNInc['](\boxlo) + \monoNNDec['](\boxlo)\bigr)\Bigr)\text{.}
    \end{equation}
    Hence, the size of the enclosure is determined by the values of the sum $\monoNNInc['] + \monoNNDec[']$ at the two corners of~$\nnInputSet$.
    We therefore prescribe the difference and the sum of the two functions: the difference $\monoNNInc['](\nnInput) - \monoNNDec['](\nnInput)$ must reproduce~$\NN(\nnInput)$, and the sum $\monoNNInc['](\nnInput) + \monoNNDec['](\nnInput)$ should equal the variation of~$\NN$ accumulated from the lower corner~$\boxlo$ to~$\nnInput$, i.e., for all $\nnInput \in \nnInputSet$,
    \begin{align}\label{eq:decomp-conditions}
      \monoNNInc['](\nnInput) - \monoNNDec['](\nnInput) & = \NN(\nnInput)\text{,} &
      \monoNNInc['](\nnInput) + \monoNNDec['](\nnInput) & = \monVar{\NN}{\shortI{\boxlo}{\nnInput}}\text{.}
    \end{align}
    The variation~$\monVar{\NN}{\shortI{\boxlo}{\nnInput}}$ in~\eqref{eq:decomp-conditions} is well defined: since $\boxlo \leq \nnInput \leq \boxhi$, each chain in~$\shortI{\boxlo}{\nnInput}$ is also a chain in~$\nnInputSet$, so that $\zeros \leq \monVar{\NN}{\shortI{\boxlo}{\nnInput}} \leq \monVar{\NN}{\nnInputSet}$ by~\eqref{eq:variation}, where~$\monVar{\NN}{\nnInputSet}$ is finite due to~\eqref{eq:chain-bound}.
    Adding the two equations in~\eqref{eq:decomp-conditions} eliminates~$\monoNNDec[']$, and subtracting the first from the second eliminates~$\monoNNInc[']$:
    \begin{align*}
      2\monoNNInc['](\nnInput)
       & = \underbrace{\bigl(\monoNNInc['](\nnInput) + \monoNNDec['](\nnInput)\bigr)}_{= \monVar{\NN}{\shortI{\boxlo}{\nnInput}}}
       + \underbrace{\bigl(\monoNNInc['](\nnInput) - \monoNNDec['](\nnInput)\bigr)}_{= \NN(\nnInput)}
       = \monVar{\NN}{\shortI{\boxlo}{\nnInput}} + \NN(\nnInput)\text{,} \\
      2\monoNNDec['](\nnInput)
       & = \underbrace{\bigl(\monoNNInc['](\nnInput) + \monoNNDec['](\nnInput)\bigr)}_{= \monVar{\NN}{\shortI{\boxlo}{\nnInput}}}
       - \underbrace{\bigl(\monoNNInc['](\nnInput) - \monoNNDec['](\nnInput)\bigr)}_{= \NN(\nnInput)}
       = \monVar{\NN}{\shortI{\boxlo}{\nnInput}} - \NN(\nnInput)\text{.}
    \end{align*}
    Dividing by two yields the unique solution
    \begin{align}\label{eq:optimal-decomp}
      \monoNNInc['](\nnInput) & = \tfrac{1}{2}\bigl(\monVar{\NN}{\shortI{\boxlo}{\nnInput}} + \NN(\nnInput)\bigr)\text{,} &
      \monoNNDec['](\nnInput) & = \tfrac{1}{2}\bigl(\monVar{\NN}{\shortI{\boxlo}{\nnInput}} - \NN(\nnInput)\bigr)\text{.}
    \end{align}

    It remains to show that~$\monoNNInc[']$ and~$\monoNNDec[']$ are non-decreasing on~$\nnInputSet$ and that they attain the lower bound.

    To show that~$\monoNNInc[']$ and~$\monoNNDec[']$ are non-decreasing on~$\nnInputSet$, let $\nnInput,\nnInput'\in\nnInputSet$ with $\nnInput \leq \nnInput'$.
    We first show that the accumulated variation grows at least by the change of~$\NN$ between~$\nnInput$ and~$\nnInput'$.
    Let $(\nnInput_0,\dots,\nnInput_p)$ be a chain in~$\shortI{\boxlo}{\nnInput}$.
    Since $\nnInput_p \leq \nnInput \leq \nnInput'$, appending~$\nnInput$ and~$\nnInput'$ yields the chain $(\nnInput_0,\dots,\nnInput_p,\nnInput,\nnInput')$ in~$\shortI{\boxlo}{\nnInput'}$, whose summed change is bounded by the supremum in~\eqref{eq:variation} for the interval~$\shortI{\boxlo}{\nnInput'}$:
    \begin{equation*}
      \sum_{\ell=1}^{p} \abs{\NN(\nnInput_\ell) - \NN(\nnInput_{\ell-1})}
      + \underbrace{\abs{\NN(\nnInput) - \NN(\nnInput_p)}}_{\geq \zeros}
      + \abs{\NN(\nnInput') - \NN(\nnInput)}
      \leq \monVar{\NN}{\shortI{\boxlo}{\nnInput'}}\text{.}
    \end{equation*}
    Omitting the non-negative middle term and taking the supremum over all chains in~$\shortI{\boxlo}{\nnInput}$ results in
    \begin{equation}\label{eq:variation-growth}
      \monVar{\NN}{\shortI{\boxlo}{\nnInput'}} - \monVar{\NN}{\shortI{\boxlo}{\nnInput}}
      \geq \abs{\NN(\nnInput') - \NN(\nnInput)}\text{.}
    \end{equation}
    Subtracting~\eqref{eq:optimal-decomp} at~$\nnInput$ from~\eqref{eq:optimal-decomp} at~$\nnInput'$ and bounding the growth of the accumulated variation by~\eqref{eq:variation-growth} yields
    \begin{align*}
      \monoNNInc['](\nnInput') - \monoNNInc['](\nnInput)
       & = \tfrac{1}{2}\Bigl(\underbrace{\monVar{\NN}{\shortI{\boxlo}{\nnInput'}} - \monVar{\NN}{\shortI{\boxlo}{\nnInput}}}_{\geq \abs{\NN(\nnInput') - \NN(\nnInput)} \text{ by \eqref{eq:variation-growth}}} + \bigl(\NN(\nnInput') - \NN(\nnInput)\bigr)\Bigr) \\
       & \geq \tfrac{1}{2}\Bigl(\abs{\NN(\nnInput') - \NN(\nnInput)} + \bigl(\NN(\nnInput') - \NN(\nnInput)\bigr)\Bigr)
       \geq \zeros\text{,} \\[0.5ex]
      \monoNNDec['](\nnInput') - \monoNNDec['](\nnInput)
       & = \tfrac{1}{2}\Bigl(\underbrace{\monVar{\NN}{\shortI{\boxlo}{\nnInput'}} - \monVar{\NN}{\shortI{\boxlo}{\nnInput}}}_{\geq \abs{\NN(\nnInput') - \NN(\nnInput)} \text{ by \eqref{eq:variation-growth}}} - \bigl(\NN(\nnInput') - \NN(\nnInput)\bigr)\Bigr) \\
       & \geq \tfrac{1}{2}\Bigl(\abs{\NN(\nnInput') - \NN(\nnInput)} - \bigl(\NN(\nnInput') - \NN(\nnInput)\bigr)\Bigr)
       \geq \zeros\text{,}
    \end{align*}
    where the last inequalities hold component-wise because $\abs{a} + a \geq 0$ and $\abs{a} - a \geq 0$ for every $a \in \mathbb{R}$.

    As~$\monoNNInc[']$ and~$\monoNNDec[']$ are non-decreasing on~$\nnInputSet$, \eqref{eq:enclosure-size-sum} holds for~$\NN'$, and we evaluate the sums at the two corners using~\eqref{eq:decomp-conditions}.
    At $\nnInput = \boxhi$, the interval~$\shortI{\boxlo}{\boxhi}$ is~$\nnInputSet$.
    At $\nnInput = \boxlo$, the interval~$\shortI{\boxlo}{\boxlo}$ only contains the point~$\boxlo$, so that every chain in it is constant, all summands in~\eqref{eq:variation} are zero, and $\monVar{\NN}{\shortI{\boxlo}{\boxlo}} = \zeros$.
    Hence,
    \begin{equation*}
      \size{\opEnclose{\NN'}{\nnInputSet}}
      = \ones^\top\Bigl(\underbrace{\monoNNInc['](\boxhi) + \monoNNDec['](\boxhi)}_{= \monVar{\NN}{\shortI{\boxlo}{\boxhi}} = \monVar{\NN}{\nnInputSet}} - \underbrace{\bigl(\monoNNInc['](\boxlo) + \monoNNDec['](\boxlo)\bigr)}_{= \monVar{\NN}{\shortI{\boxlo}{\boxlo}} = \zeros}\Bigr)
      = \ones^\top\monVar{\NN}{\nnInputSet}\text{.}
    \end{equation*}
    Thus, $\monoNNInc[']$ and~$\monoNNDec[']$ are non-decreasing on~$\nnInputSet$, satisfy $\monoNNInc['] - \monoNNDec['] = \NN$ on~$\nnInputSet$, and attain the lower bound, which proves the second statement.
  \end{proof}
\end{repproposition}

\subsection{Benchmark definitions}\label{sec:benchmarks}
In this subsection, we define the scalar analytic model, the two dynamical systems, and the multi-input analytic model.

\paragraph{Scalar analytic model} The neural network is trained to approximate the model
\begin{equation}
  f(\nnInput) = \sin \nnInput + 0.3 \sin 3\nnInput + 0.05\nnInput
\end{equation}
on~$\shortI{-3}{3}$.
This model is not monotone, as it has six local extrema in~$\shortI{-3}{3}$.

\paragraph{Dynamical systems}
A system $\dot{\nnInput} = f(\nnInput,u)$ is cooperative if it satisfies the Kamke--M\"uller condition~\citep{smith_1995}
\begin{equation}
  \frac{\partial f_{(i)}}{\partial \nnInput_{(j)}} \geq 0
  \ \ \forall i \neq j\text{,}
  \qquad
  \frac{\partial f}{\partial u} \geq 0\text{.}
  \label{eq:kamke}
\end{equation}
Every cooperative system is also a monotone system in the sense that, for a constant input, all solutions starting in an interval~$\shortI{\boxlo}{\boxhi}$ are bounded by the two solutions starting in~$\boxlo$ and~$\boxhi$~\citep{smith_1995}.

The Pendulum system uses the dynamics of the Gymnasium environment \texttt{Pendulum-v1}~\citep{towers_2024}; due to the $\sin$ term, it violates~\eqref{eq:kamke} and is thus non-cooperative.
The Tanks system consists of two cascaded tanks with the fill levels~$\nnInput_{(1)}$ and~$\nnInput_{(2)}$ and the inflow~$u$:
\begin{align}
  \begin{split}
    \dot{\nnInput}_{(1)} & = -k_1\sqrt{\nnInput_{(1)}} + k_4 u\text{,} \\
    \dot{\nnInput}_{(2)} & = k_2\sqrt{\nnInput_{(1)}} - k_3\sqrt{\nnInput_{(2)}}\text{,}
  \end{split}
\end{align}
with $k_1 = 0.5$, $k_2 = 0.4$, $k_3 = 0.3$, $k_4 = 0.12$, and~$u\in\shortI{0}{10}$.
The system satisfies~\eqref{eq:kamke} and is thus cooperative.
We integrate the system with a fourth-order Runge--Kutta scheme and use a time step of~$4\,\mathrm{s}$.

For the prediction accuracy~(\cref{tab:accuracy-main}), the neural network is trained to predict the change of the state within one time step from the current state and the input; for the verification-aware training~(\cref{tab:beta}) and the evaluation over multiple time steps~(\cref{sec:closedloop}), which both use the Tanks system, it is trained to predict the next state.
For these two evaluations, we predict the next state, because the Tanks system is cooperative, so that the map from the current state to the next state is monotone for a constant input, whereas the change of the state does not need to be monotone.
For the Tanks system with eight layers, the tightness~(\cref{sec:eval}) of our enclosure is~$1.08$ for the next state and~$1.37$ for the change of the state ($\epsilon = 0.05$, ten seeds).

\paragraph{Multi-input analytic model}
The neural network is trained to approximate the model
\begin{equation}\label{eq:synthetic}
  f(\nnInput) = \sin \nnInput_{(1)} + 0.3 \sin 3\nnInput_{(2)} + \frac{0.05}{\numNeurons_0}\sum_{i=1}^{\numNeurons_0} \nnInput_{(i)}
\end{equation}
with~$\numNeurons_0$ inputs in~$\shortI{-3}{3}$, of which only~$\nnInput_{(1)}$ and~$\nnInput_{(2)}$ enter nonlinearly.
We perturb either all inputs by the same radius, so that the size of the input set increases with~$\numNeurons_0$, or only~$\nnInput_{(1)}$ and~$\nnInput_{(2)}$, so that the size of the input set does not depend on~$\numNeurons_0$.

\subsection{Computation of the tightness}
\label{sec:metrics}

Since the exact image~$\NN(\nnInputSet)$ is not available, we evaluate the neural network at uniformly sampled inputs in~$\nnInputSet$ and at the corners of~$\nnInputSet$ (for~$\numNeurons_0 > 12$, only at~$\boxlo$ and~$\boxhi$).
The resulting outputs are a subset of~$\NN(\nnInputSet)$; thus, their size is at most~$\size{\NN(\nnInputSet)}$.
As we divide the size of the enclosure by the size of the sampled outputs, the reported tightness is an upper bound of the tightness with respect to the exact image.
For the dynamical systems, the inputs and outputs of the neural network are standardized using the mean and the standard deviation of the training data; the perturbation radii refer to the standardized inputs, and the sizes of the enclosures and of the sampled outputs are computed in the original units of the outputs.
For the static analytic models, the inputs and outputs are not rescaled.
For each seed, we first average the sizes of the enclosures and the sizes of the sampled outputs over all input sets and then compute their ratio, so that input sets with a small image do not dominate the result; we report the mean and the standard deviation of this ratio over the seeds.
We mark a configuration with~$\dagger$ if the prediction accuracy of at least one seed is below~$R^2 = 0.9$; we do not use marked configurations to compare architectures or verification methods.
The numbers of input sets and seeds of each experiment are listed in~\cref{tab:settings}, and the runtime measurements are described in~\cref{app:timing}.

\subsection{Experimental details}
\label{app:details}

\paragraph{Architectures}
Each of our sign-constrained subnetworks has~$32$ neurons per hidden layer, of which we assign one third to each of the activation functions~$\relu(\nnInput)$, $-\relu(-\nnInput)$, and~$\tanh(\nnInput)$.
For multiple outputs, each output has its own pair of subnetworks.
We compare our mixed-monotone neural networks with unconstrained neural networks (baseline).
Additionally, we use the monotone Lipschitz-constrained neural networks of~\citet{kitouni_2023} as subnetworks in~\eqref{eq:arch} (Lip.-res.), where each subnetwork has the form $\nnInput \mapsto \lipBudget\,\ones^\top\nnInput + h(\nnInput)$ with a neural network~$h$ whose Lipschitz constant with respect to the $\ell_1$-norm is at most~$\lipBudget$, and we consider non-monotone neural networks that use the Lipschitz-constrained layers of~\citet{kitouni_2023} without the linear function (Lip.).
The number of parameters of each architecture differs from that of our neural networks by at most~$2\,\%$; to this end, we adapt the number of neurons per hidden layer of the baseline and of the non-monotone Lipschitz-constrained neural networks.
Within each experiment, all architectures use the same data splits, optimizer settings, and number of epochs, which we did not tune per architecture.

\paragraph{Initialization}
We initialize each weight of a sign-constrained linear layer with the absolute value of a sample from a standard normal distribution, which we scale such that the expected value of each weight is~$c/\numNeurons_{k-1}$ with~$c = 1.4$.
In contrast to weights with zero mean, non-negative weights do not cancel each other in the weighted sum of a neuron: for weights that are independent of the inputs, the expected weighted sum of a neuron is $\numNeurons_{k-1}$ times the expected weight times the mean of its inputs.
For the expected weight~$c/\numNeurons_{k-1}$, this expected value does not depend on~$\numNeurons_{k-1}$, whereas it is larger by the factor~$\sqrt{\numNeurons_{k-1}}$ for the common scaling by~$1/\sqrt{\numNeurons_{k-1}}$, which can accumulate over the layers.
In our experiments, the magnitude of the outputs of all layers remains within~$\shortI{0.3}{6}$ for the evaluated numbers of layers.
The training is sensitive to~$c$, e.g., for~$c = 5$, the neural network does not fit the scalar analytic model.

\paragraph{Training}
We train all neural networks with Adam using a learning rate of~$2\cdot10^{-3}$ and a batch size of~$64$, without learning-rate decay and without early stopping; the number of epochs is listed in~\cref{tab:settings}.
For~$\widthWeight > 0$, we increase the weight of the enclosure size in the training loss~\eqref{eq:loss} in three stages of the training from~$0$ to~$\widthWeight/2$ to~$\widthWeight$, so that the neural network fits the data before the enclosure is tightened; the state of the optimizer is kept between the stages.
In each training iteration, we evaluate the regression loss for~$64$ training inputs and the enclosure size for~$8$ input sets with the radius~$\epsilon = 0.2$~(\cref{sec:metrics}), which are centered at randomly drawn training inputs and clipped to the training domain.
The~$64$ training inputs and the~$16$ corners of the input sets are evaluated in a single forward pass of one batch.
The input sets are drawn once per stage, resulting in~$\lfloor N/64 \rfloor \cdot 8$ different input sets for~$N$ training inputs.

\paragraph{Non-negative weights}\label{app:nonnegative-weights}
The sign-constrained subnetworks require non-negative weights in the forward pass~(\cref{def:neural_network}).
We use~$\abs{\Weights}$ in the forward pass and optimize the unconstrained weights~$\Weights$, so that the gradient of the training loss follows from the chain rule:
\begin{equation}
  \frac{\partial \pointLoss}{\partial \Weights}
  = \frac{\partial \pointLoss}{\partial \abs{\Weights}}
  \odot \sign{\Weights}\text{,}
  \label{eq:reparam}
\end{equation}
where~$\odot$ denotes the elementwise product; for entries that are zero, where the absolute value is not differentiable, we use~$+1$, which is an element of the subdifferential~$\shortI{-1}{1}$ of the absolute value at zero; with the subgradient~$0$, a weight that is exactly zero would not be updated anymore.
Hence, the weights used in the forward pass are non-negative throughout the training, which preserves the monotonicity of both subnetworks.
An alternative is to project the weights onto the non-negative values during the training, i.e., $\Weights \leftarrow \max(\Weights,\zeros)$.
In a preliminary experiment on the scalar analytic model, in which we applied this projection every~$30$ epochs, $32.3\,\%$ of the weights were exactly zero after~$120$ epochs and the prediction accuracy decreased from~$R^2 = 0.98$ to~$0.76$; therefore, we use the absolute value instead.

\paragraph{Settings of the experiments}
\cref{tab:settings} lists the settings of all experiments.
The neural networks for the comparison of verification methods~(\cref{tab:verifiers-main,tab:lirpa}) and for the comparison across the number of layers~(\cref{tab:crossover}) are trained separately, both with~$\widthWeight = 0$; therefore, the tightness of our enclosure in~\cref{tab:verifiers-main} and~\cref{tab:crossover} differs for the same number of layers.
The tightness of an enclosure is estimated with~$800$ samples per input set ($2000$ for the evaluation over multiple time steps).
All random choices of an experiment, i.e., the dataset, its split, the initialization of the weights, the input sets, and the samples, are derived from a single seed.
For the dynamical systems, we split the dataset by trajectories, so that the states of one trajectory are not distributed over the training, validation, and test data.
\begin{table}[h]
  \centering
  \caption{Settings of the experiments.}
  \label{tab:settings}
  \setlength{\tabcolsep}{5pt}
  \begin{tabular}{l l c c c c c}
    \toprule
    \textbf{Experiment} & \textbf{Benchmark} & \textbf{\#Layers}~$\numLayers$ & $\widthWeight$ & \textbf{\#Seeds} & \textbf{\#Epochs} & \textbf{\#Input Sets} \\
    \midrule
    \cref{tab:accuracy-main,tab:accuracy} & Scalar analytic & 2, 3, 4, 6, 8 & 0    & 10 & 120 & --  \\
    \cref{tab:accuracy-main}              & Pendulum, Tanks & 3, 8          & 0    & 10 & 100 & --  \\
    \cref{tab:verifiers-main,tab:lirpa}   & Scalar analytic & 2, 3, 4, 6, 8 & 0    & 5  & 120 & 100 \\
    \cref{tab:beta}                       & Scalar analytic, Tanks & 3      & 0--1 & 3  & 80  & 150 \\
    \cref{tab:dims-main,tab:dims}         & Multi-input analytic & 4        & 0.05 & 5  & 120 & 100 \\
    \cref{tab:crossover}                  & Scalar analytic & 2, 3, 4, 6, 8 & 0    & 10 & 120 & 150 \\
    \cref{fig:closedloop}                 & Tanks           & 3, 8          & 0.05 & 3  & 100 & 20  \\
    \bottomrule
  \end{tabular}
\end{table}

\subsection{Additional results on the prediction accuracy}
\label{app:accuracy}

\Cref{tab:accuracy} shows the prediction accuracy on the scalar analytic model for all tested numbers of layers and architectures.
The baseline remains accurate for all numbers of layers, and the accuracy of our architecture differs from it by at most~$0.01$.
In contrast, both Lipschitz-constrained architectures lose accuracy at six and eight layers when trained with the same settings: at eight layers, the monotone Lipschitz-constrained architecture reaches~$R^2=0.752\pm0.048$ and the non-monotone one~$R^2=0.635\pm0.074$.

\begin{table}[h]
  \centering
  \caption{Comparing the prediction accuracy of different architectures on the scalar analytic model shown as mean $\pm$ standard deviation over ten seeds; bold font marks the best architecture per number of layers.}
  \label{tab:accuracy}
  \begin{tabular}{c R@{$\,\pm\,$}L R@{$\,\pm\,$}L R@{$\,\pm\,$}L R@{$\,\pm\,$}L}
    \toprule
    \multirow{2}{*}{\textbf{\#Layers}~$\numLayers$}
     & \multicolumn{8}{c}{\textbf{Architecture}} \\
    \cmidrule(lr){2-9}
     & \multicolumn{2}{c}{\makecell{Baseline \\(unconstr.)}}
     & \multicolumn{2}{c}{\makecell{Ours \\(sign-constr)}}
     & \multicolumn{2}{c}{\makecell{Lip.-res. (monotone) \\\citet{kitouni_2023}}}
     & \multicolumn{2}{c}{\makecell{Lip. \\(non-monotone)}} \\
    \midrule
    2
     & 0.984          & 0.002 & \mathbf{0.986} & 0.001 & 0.973 & 0.002 & 0.972 & 0.002 \\
    3
     & \mathbf{0.999} & 0.002 & 0.992          & 0.003 & 0.983 & 0.001 & 0.983 & 0.001 \\
    4
     & \mathbf{1.000} & 0.000 & 0.990          & 0.005 & 0.982 & 0.001 & 0.981 & 0.001 \\
    6
     & \mathbf{1.000} & 0.000 & 0.992          & 0.006 & 0.876 & 0.022 & 0.838 & 0.028 \\
    8
     & \mathbf{1.000} & 0.000 & 0.996          & 0.004 & 0.752 & 0.048 & 0.635 & 0.074 \\
    \bottomrule
  \end{tabular}
\end{table}

\subsection{Additional results on the comparison of verification methods}
\label{app:crown}

\Cref{tab:lirpa} extends~\cref{tab:verifiers-main} to all tested numbers of layers.
CROWN is tighter than our enclosure at two and three layers, and our enclosure is tighter at six and eight layers; at four layers, the difference is smaller than the standard deviation over the seeds.
The tightness of $\alpha$-CROWN, which optimizes its bounds with~$20$ iterations and a learning rate of~$0.1$, is at most~$1.004$ for all numbers of layers.

\begin{table}[t]
  \centering
  \papertablestyle[2pt]
  \caption{Comparing enclosure tightness and runtimes of all verification methods on the scalar analytic model at~$\epsilon = 0.05$ for all tested numbers of layers.
    The tightness is shown as mean $\pm$ standard deviation over five seeds, and the runtimes are relative to our enclosure within the same toolbox; bold font marks the best tightness per number of layers.}
  \label{tab:lirpa}
  \resizebox{\textwidth}{!}{%
    \begin{tabular}{c R@{$\,\pm\,$}L R@{$\,\pm\,$}L R@{$\,\pm\,$}L R@{$\,\pm\,$}L R R R R}
      \toprule
      \multirow{2}{*}{\textbf{\#Layers}~$\numLayers$}
        & \multicolumn{8}{c}{\textbf{Tightness}~$\downarrow$}
        & \multicolumn{4}{c}{\textbf{Runtime relative to ours}~$\downarrow$} \\
      \cmidrule(lr){2-9} \cmidrule(lr){10-13}
        & \multicolumn{2}{c}{Ours / IBP} & \multicolumn{2}{c}{Zonotope}
        & \multicolumn{2}{c}{CROWN}      & \multicolumn{2}{c}{$\alpha$-CROWN}
        & \text{IBP} & \text{Zono.} & \text{CROWN} & \text{$\alpha$-CR.} \\
      \midrule
      2 & 3.58 & 0.55 & 1.08 & 0.01 & 1.07 & 0.01 & \mathbf{1.001} & 0.000
        & 10.8 & 6.5  & 23.5 & 1\,490 \\
      3 & 2.00 & 0.43 & 1.14 & 0.03 & 1.12 & 0.03 & \mathbf{1.004} & 0.002
        & 8.1  & 8.4  & 54.4 & 6\,130 \\
      4 & 1.17 & 0.12 & 1.16 & 0.05 & 1.13 & 0.04 & \mathbf{1.002} & 0.002
        & 9.9  & 11.7 & 112  & 8\,750 \\
      6 & 1.07 & 0.08 & 1.25 & 0.10 & 1.18 & 0.04 & \mathbf{1.002} & 0.001
        & 10.2 & 19.6 & 193  & 12\,900 \\
      8 & 1.05 & 0.08 & 1.44 & 0.07 & 1.26 & 0.06 & \mathbf{1.002} & 0.001
        & 11.2 & 31.8 & 256  & 20\,600 \\
      \bottomrule
    \end{tabular}%
  }
\end{table}

\paragraph{Interval bound propagation}
For our architecture, IBP returns the same enclosure as~\cref{prop:enclosure}.
IBP propagates an interval~$\shortI{l}{u}$ through a linear layer with weights~$\Weights$ and bias~$\bias$ to the interval $\shortI{\max(\Weights,\zeros)\,l + \min(\Weights,\zeros)\,u + \bias}{\max(\Weights,\zeros)\,u + \min(\Weights,\zeros)\,l + \bias}$, where the maximum and the minimum are taken element-wise~\citep{gowal_2019}.
For non-negative weights, we have $\min(\Weights,\zeros) = \zeros$, so that the resulting interval is $\shortI{\Weights l + \bias}{\Weights u + \bias}$, i.e., the layer is evaluated at the lower and the upper bound of the input interval.
A non-decreasing activation function maps an interval~$\shortI{l}{u}$ to the interval between the function values at~$l$ and~$u$.
By induction over the layers, IBP thus propagates the input interval~$\shortI{\boxlo}{\boxhi}$ through~$\monoNNInc$ and~$\monoNNDec$ to the intervals~$\shortI{\monoNNInc(\boxlo)}{\monoNNInc(\boxhi)}$ and~$\shortI{\monoNNDec(\boxlo)}{\monoNNDec(\boxhi)}$.
Subtracting the second interval from the first one in interval arithmetic yields $\shortI{\monoNNInc(\boxlo) - \monoNNDec(\boxhi)}{\monoNNInc(\boxhi) - \monoNNDec(\boxlo)} = \opEnclose{\NN}{\nnInputSet}$.
The equality holds because all weights of the subnetworks are non-negative: for the Lipschitz-constrained subnetworks of~\citet{kitouni_2023}, whose weights are not sign-constrained, the size of the IBP enclosure is larger than the size of our enclosure by a factor of~$1.49$ at two layers and~$1.15$ at eight layers ($\epsilon = 0.05$, ten seeds).
We verified the equality numerically: in CORA, the sizes of both enclosures agree up to the reported precision, and for the neural networks exported to \texttt{auto\_LiRPA}, the relative difference of the sizes is at most~$1.1 \cdot 10^{-6}$ with a median of~$6 \cdot 10^{-8}$ over all exported neural networks and tested perturbation radii.

\paragraph{Export to \texttt{auto\_LiRPA}}\label{app:lirpa}
We compute the bounds of IBP, CROWN, and $\alpha$-CROWN with \texttt{auto\_LiRPA}~v0.7.2 (PyTorch~2.11.0, Python~3.11), which requires a PyTorch model; therefore, we export the neural networks trained in CORA to PyTorch.
The exported model reproduces every layer of the CORA neural network: the weights~$\Weights_k$ of the linear layers enter the forward pass as~$\abs{\Weights_k}$, and the nonlinear layers apply the three activation functions as described in~\cref{app:details}.
The subtraction~$\NN = \monoNNInc - \monoNNDec$ is implemented as a final linear layer with the fixed weights~$\cmat{I \ {-I}}$, where~$I$ is the identity matrix, so that all methods bound the output of~$\NN$ instead of bounding the two subnetworks separately.
The weights are transferred from the neural networks trained in CORA without retraining.
For every exported neural network, we check that its outputs match the outputs of the CORA neural network for~$512$ uniformly sampled inputs and that IBP in \texttt{auto\_LiRPA} returns our enclosure.

\paragraph{Runtime measurements}\label{app:timing}
All neural networks are trained on a Linux cluster.
The runtimes in CORA~(\cref{tab:verifiers-main,tab:lirpa,app:dims}) are measured on an otherwise idle machine with an Intel Core Ultra~9 285K ($24$ cores, $62$\,GB) using MATLAB R2026a and CORA v2026.1.0.
For each configuration, we measure the runtime for~$200$ input sets and report the minimum over~$25$ repetitions after one warm-up evaluation, as the mean is dominated by isolated outliers; the interquartile range is within~$10\,\%$ of the median for our enclosure and within~$20\,\%$ for zonotope propagation.
As the runtime does not depend on the values of the weights, we use randomly initialized neural networks for these measurements.
The runtimes in \texttt{auto\_LiRPA} are measured on a different machine with an AMD Ryzen~7 Microsoft Surface Edition (Zen~2, $8$ cores, $16$\,GB) and PyTorch restricted to eight threads; we report the minimum over ten repetitions after one warm-up evaluation and exclude the construction of the \texttt{BoundedModule}, which is only required once per neural network.
Therefore, we only compare runtimes that are measured within the same toolbox.
Although IBP computes the same bounds as our enclosure, IBP in \texttt{auto\_LiRPA} is eight to eleven times slower than our two forward passes for all tested numbers of layers and perturbation radii; our measurements do not separate the overhead of the toolbox from the cost of the bound computation.

\subsection{Additional results on the number of layers}
\label{app:crossover-results}

This subsection qualifies the decrease of the crossover radius with the number of layers in~\cref{tab:crossover}.
Between six and eight layers, the medians of the crossover radius differ by one step of the tested perturbation radii, and the ranges of the crossover radius over the seeds overlap; hence, our measurements do not show a further decrease of the crossover radius from six to eight layers.
At eight layers, the ratio of the size of our enclosure to the perturbation radius decreases from~$1.448$ to~$1.219$ over two orders of magnitude of the perturbation radius, while this ratio increases by a factor of twelve for zonotope propagation.

\subsection{Additional results on the input dimension}
\label{app:dims}

\Cref{tab:dims} shows the enclosure tightness for all tested input dimensions and for both ways of perturbing the inputs~(\cref{sec:benchmarks}).
Our architecture with sign-constrained subnetworks (Ours) and with Lipschitz-constrained subnetworks (Lip.-res.) is enclosed using~\cref{prop:enclosure}, and the baseline is enclosed using zonotope propagation.
If only the first two inputs are perturbed, the interval norm of the input set is approximately~$1.9$ for all input dimensions, and the tightness of our enclosure remains below~$1.05$.
If all inputs are perturbed, the interval norm of the input set and the tightness of all enclosures increase with the input dimension.
In CORA, computing our enclosure takes~$13$ to~$16\,\mu\mathrm{s}$ per input set for all tested input dimensions, while zonotope propagation through the baseline takes~$4.5$ to~$7.5$ times as long~(\cref{app:timing}).

\begin{table}[h]
  \centering
  \papertablestyle
  \caption{Comparing the enclosure tightness for different input dimensions~$\numNeurons_0$ at~$\epsilon = 0.5$, when all inputs or only the first two inputs are perturbed, shown as mean $\pm$ standard deviation over five seeds; $\dagger$ marks configurations with $R^2 < 0.9$~(\cref{sec:metrics}).}
  \label{tab:dims}
  \begin{tabular}{l R R R@{$\,\pm\,$}L R@{$\,\pm\,$}L R@{$\,\pm\,$}L}
    \toprule
    \textbf{Perturbed}
     & \numNeurons_0
     & \textbf{Input Set Size}~$\size{\nnInputSet}$
     & \multicolumn{2}{c}{\textbf{Ours}}
     & \multicolumn{2}{c}{\makecell{\textbf{Lip.-res.} \\\citet{kitouni_2023}}}
     & \multicolumn{2}{c}{\textbf{Zonotope}} \\
    \midrule
    \multirow{6}{*}{all}
     & 2   & 1.92  & 1.045 & 0.008 & 19.4         & 0.4 & 1.65         & 0.01 \\
     & 5   & 4.82  & 1.054 & 0.012 & 38.2         & 3.2 & 1.59         & 0.05 \\
     & 10  & 9.56  & 1.050 & 0.009 & 86.5^\dagger & 6.9 & 1.66         & 0.08 \\
     & 20  & 19.18 & 1.073 & 0.009 & 181^\dagger  & 37  & 1.77         & 0.09 \\
     & 50  & 47.90 & 1.105 & 0.013 & 378^\dagger  & 38  & 152^\dagger  & 12   \\
     & 100 & 95.91 & 1.167 & 0.039 & 862^\dagger  & 95  & 360^\dagger  & 31   \\
    \midrule
    \multirow{6}{*}{first two}
     & 2   & 1.92  & 1.045 & 0.008 & 19.4         & 0.4 & 1.65         & 0.01 \\
     & 5   & 1.93  & 1.047 & 0.011 & 16.1         & 1.3 & 1.54         & 0.05 \\
     & 10  & 1.90  & 1.041 & 0.006 & 19.3^\dagger & 1.8 & 1.55         & 0.07 \\
     & 20  & 1.91  & 1.045 & 0.008 & 20.3^\dagger & 4.9 & 1.52         & 0.05 \\
     & 50  & 1.91  & 1.040 & 0.005 & 16.8^\dagger & 1.8 & 2.42^\dagger & 0.16 \\
     & 100 & 1.92  & 1.028 & 0.007 & 18.9^\dagger & 2.0 & 2.60^\dagger & 0.06 \\
    \bottomrule
  \end{tabular}
\end{table}

\subsection{Details on the evaluation over multiple time steps}
\label{app:closedloop-details}

This subsection provides the details of the evaluation over multiple time steps in~\cref{sec:closedloop} and examines whether the divergence of the zonotope enclosure is caused by the order reduction of the zonotopes.
Starting from an initial input set~$\nnInputSet_0 = \shortI{\boxlo_0}{\boxhi_0}$, we enclose the reachable set of the next time step using~\cref{prop:enclosure}:
\begin{equation}
  \nnInputSet_{t+1}
  = \shortI{\boxlo_{t+1}}{\boxhi_{t+1}}
  = \shortI{\decompFun{\NN}(\boxlo_t,\boxhi_t)}
  {\decompFun{\NN}(\boxhi_t,\boxlo_t)}\text{,}
  \label{eq:closedloop}
\end{equation}
where the input of the system is constant.
As baseline, we propagate a zonotope through an unconstrained neural network that is trained on the same data: the output zonotope of one time step is used as the input set of the next time step without enclosing it by an interval, and its order is reduced to~$20$ after every time step using the method of~\citet{girard_2005}.
This evaluation therefore compares two combinations of neural network architecture and verification method; the comparison of verification methods on the same neural network is given in~\cref{tab:verifiers-main}.

For each of the two numbers of layers~$\numLayers \in \{3, 8\}$, we train three neural networks (seeds) and start from~$20$ initial input sets with a radius of~$0.2$ per neural network, which yields~$60$ runs per number of layers, each with at most~$40$ time steps and~$2000$ sampled trajectories.
In every time step, we check that all samples are contained in all enclosures up to a tolerance of~$10^{-6}$ plus~$10^{-3}$ times the largest edge length of the sampled reachable set; this check did not fail in any run.
A run ends when the size of the sampled reachable set is below~$2\cdot10^{-5}$, when the zonotope propagation fails or returns non-finite values, or after~$40$ time steps.
\cref{fig:closedloop} shows the median of the enclosure sizes over all runs in every time step.
A run whose zonotope propagation failed or returned non-finite values is counted with an infinite size in all later time steps, so that the median remains finite as long as fewer than half of the runs have ended in this way; a run that ended because the sampled reachable set contracted is not counted in later time steps.

At three layers, all zonotope enclosures remain finite, and the tightness of the zonotope enclosure increases from~$1.3$ to~$4.3$ within~$33$ time steps with a maximum of~$6.9$.
At eight layers, the median size of the zonotope enclosure exceeds~$10^6$ after eleven time steps and reaches~$2.9\cdot10^{37}$ after~$30$ time steps; in~$26$ of the~$60$ runs, the zonotope propagation fails after the size of the zonotope enclosure has exceeded~$8\cdot10^{37}$ or become non-finite, $41$ runs exceed a size of~$10^3$, and~$19$ runs stay below a size of~$10$.
Within the evaluated time steps, the sampled reachable set contracts by more than four orders of magnitude.

To examine whether this divergence is caused by the order reduction, we repeat the evaluation for the baselines with eight layers of the Tanks and the Pendulum system with five initial input sets with the radius~$0.2$ per seed, $800$ samples per input set, and the zonotope orders~$10$, $20$, $50$, $100$, and~$200$.
Increasing the order from~$20$ to~$200$ reduces the size of the zonotope enclosure after~$20$ time steps by about one order of magnitude, with little change between the orders~$100$ and~$200$.
However, the median of the first time step at which the zonotope enclosure is more than~$100$ times larger than the sampled reachable set does not change for either system.
These results only hold for the trained neural networks; enclosing the states of the physical system would additionally require taking the error of the model into account.

\end{document}